\documentclass{vldb}
\usepackage[utf8]{inputenc} 
\usepackage[T1]{fontenc}    
\usepackage{hyperref}       
\usepackage{url}            
\usepackage{booktabs}       
\usepackage{amsfonts}       
\usepackage{nicefrac}       
\usepackage{microtype}      
\usepackage{graphicx}       
\usepackage{algorithm}      
\usepackage{algorithmic}
\usepackage{epstopdf}
\usepackage{sidecap}
\usepackage{color}
\usepackage{balance}  
\newtheorem{theorem}{Theorem}
\newtheorem{corollary}{Corollary}

\begin{document}


\title{Scalable and Sustainable Deep Learning\\ via Randomized Hashing}



%
%
%
%

\numberofauthors{2} 

\author{
%
%
\alignauthor
Ryan Spring\\
       \affaddr{Rice University}\\
       \affaddr{Department of Computer Science}\\
       \affaddr{Houston, TX, USA}\\
       \email{rdspring1@rice.edu}
\alignauthor
Anshumali Shrivastava\\
       \affaddr{Rice University}\\
       \affaddr{Department of Computer Science}\\
       \affaddr{Houston, TX, USA}\\
       \email{anshumali@rice.edu}
}

\maketitle

\begin{abstract}
Current deep learning architectures are growing larger in order to learn from complex datasets. These architectures require giant matrix multiplication operations to train millions of parameters. Conversely, there is another growing trend to bring deep learning to low-power, embedded devices. The matrix operations, associated with both training and testing of deep networks, are very expensive from a computational and energy standpoint. We present a novel hashing based technique to drastically reduce the amount of computation needed to train and test deep networks. Our approach combines recent ideas from adaptive dropouts and randomized hashing for maximum inner product search to select the nodes with the highest activation efficiently. Our new algorithm for deep learning reduces the overall computational cost of forward and back-propagation by operating on significantly fewer (sparse) nodes. As a consequence, our algorithm uses only 5\% of the total multiplications, while keeping on average within 1\% of the accuracy of the original model. A unique property of the proposed hashing based back-propagation is that the updates are always sparse. Due to the sparse gradient updates, our algorithm is ideally suited for asynchronous and parallel training leading to near linear speedup with increasing number of cores. We demonstrate the scalability and sustainability (energy efficiency) of our proposed algorithm via rigorous experimental evaluations on several real datasets. 
\end{abstract}

\section{Introduction}
Deep learning is revolutionizing big-data applications, after being responsible for groundbreaking improvements in image classification \cite{krizhevsky2012imagenet} and speech recognition \cite{hinton2012deep}.  With the recent upsurge in data, at a much faster rate than our computing capabilities, neural networks are growing larger to process information more effectively. In 2012, state-of-the-art convolutional neural networks contained at most 10 layers. Afterward, each successive year has brought deeper architectures with greater accuracy. Last year, Microsoft's deep residual network \cite{he2015deep}, which won the ILSVRC 2015 competition with a 3.57\% error rate, had 152 layers and 11.3 billion FLOPs. To handle such large neural networks, researchers usually train them on high-performance graphics cards or large computer clusters. 

The basic building block of a neural network is a neuron. Each neuron activates when a specific feature appears in its input. The magnitude of a neuron's activation is a measure of the neuron's confidence for the presence or absence of a feature. In the first layer of deep network, the neurons detect simple edges in an image. Each successive layer learns more complex features. For example, the neurons will find parts of a face first, such as the eyes, ears, or mouth, before distinguishing between different faces. For classification tasks, the final layer is a classic linear classifier - Softmax or SVM. Multiple layers combine to form a complex, non-linear function capable of representing any arbitrary function \cite{cybenko1989approximation}. Adding more neurons to a network layers increases its expressive power. This enhanced expressive power has made massive-sized deep networks a common practice in large scale machine learning systems, which has shown some very impressive boost over old benchmarks.  

Due to the growing size and complexity of networks, efficient algorithms for training massive deep networks in a distributed and parallel environment is currently the most sought after problem in both academia and the commercial industry. For example, Google \cite{dean2012large} used a 1-billion parameter neural network, which took three days to train on a 1000-node cluster, totaling over 16,000 CPU cores. Each instantiation of the network spanned 170 servers. In distributed computing environments, the parameters of giant deep networks are required to be split across multiple nodes. However, this setup requires costly communication and synchronization between the parameter server and processing nodes in order to transfer the gradient and parameter updates. The sequential and dense nature of gradient updates prohibits any efficient splitting (sharding) of the neural network parameters across computer nodes. There is no clear way to avoid the costly synchronization without resorting to some ad-hoc breaking of the network. This ad-hoc breaking of deep networks is not well understood and is likely to hurt performance. Synchronization is one of the major hurdles in scalability. Asynchronous training is the ideal solution, but it is sensitive to conflicting, overlapping parameter updates, which leads to poor convergence.

While deep networks are growing larger and more complex, there is also push for greater energy efficiency to satisfy the growing popularity of machine learning applications on mobile phones and low-power devices. For example, there is recent work by \cite{mcmahan2016federated} aimed at leveraging the vast data of mobile devices. This work has the users train neural networks on their local data, and then periodically transmit their models to a central server. This approach preserves the privacy of the user's personal data, but still allows the central server's model to learn effectively. Their work is dependent on training neural networks locally. Back-propagation~\cite{rumelhart1988learning} is the most popular algorithm for training deep networks. Each iteration of the back-propagation algorithm is composed of giant matrix multiplications. These matrices are very large, especially for massive networks with millions of nodes in the hidden layer, which are common in industrial practice. Large matrix multiplications are parallelizable on GPUs, but not energy-efficient. Users require their phones and tablets to have long battery life. Reducing the computational costs of neural networks, which directly translates into longer battery life, is a critical issue for the mobile industry.

The current challenges for deep learning illustrate a great demand for algorithms that reduce the amount of computation and energy usage. To reduce the bottleneck matrix multiplications, there has been a flurry of works around reducing the amount of computations associated with them. Most of them revolve around exploiting low-rank matrices or low precision updates. We review these techniques in details in Section~\ref{sec:related_work}. However, updates with these techniques are hard to parallelize making them unsuitable for distributed and large scale applications. On the contrary, our proposal capitalizes on the sparsity of the activations to reduce computations. To the best of our knowledge, this is the first proposal that exploits sparsity to reduce the amount of computation associated with deep networks. We further show that our approach admits asynchronous parallel updates leading to perfect scaling with increasing parallelism. 

Recent machine learning research has focused on techniques for dealing with the famous problem of over-fitting with deep networks. A notable line of work \cite{ba2013adaptive, makhzani2013k, makhzani2015winner} improved the accuracy of neural networks by only updating the neurons with the highest activations. Adaptive dropout \cite{ba2013adaptive} sampled neurons in proportion to an affine transformation of the neuron's activation. The Winner-Take-All (WTA) approach \cite{makhzani2013k, makhzani2015winner} kept the top-k\% neurons, by calculating a hard threshold from mini-batch statistics. It was found that such a selective choice of nodes and sparse updates provide a natural regularization \cite{srivastava2014dropout}. However, these approaches rely on inefficient, brute-force techniques to find the best neurons. Thus, these techniques are equally as expensive as the standard back-propagation method, leading to no computational savings. 

We propose a hashing-based indexing approach to train and test neural networks, which capitalizes on the rich theory of randomized sub-linear algorithms from the database literature to perform adaptive dropouts \cite{ba2013adaptive} efficiently while requiring significantly less computation and memory overhead. Furthermore, hashing algorithms are embarrassingly parallel and can be easily distributed with small communications, which is perfect for large-scale, distributed deep networks.

Our idea is to index the neurons (the weights of each neuron as a vector) into a hash table using locality sensitive hashing. Such hash tables of neurons allow us to select the neurons with the highest activations, without computing all of the activations, in the sub-linear time leading to significant computational savings. Moreover, as we show, since our approach results in a sparse active set of neurons randomly, the gradient updates are unlikely to overwrite. Such updates are ideal for asynchronous and parallel gradient updates. It is known that asynchronous stochastic gradient descent\cite{recht2011hogwild} (ASGD) will converge if the number of simultaneous parameter updates is small. We heavily leverage this sparsity which unique to our proposal. On several deep learning benchmarks, we show that our approach outperforms standard algorithms including vanilla dropout \cite{srivastava2014dropout} at high sparsity levels and matches the performance of adaptive dropout \cite{ba2013adaptive} and winner-take-all \cite{makhzani2013k, makhzani2015winner} while needing less computation (only 5\%).

\subsection{Our Contributions}
\begin{enumerate} 
    \item We present a scalable and sustainable algorithm for training and testing of fully-connected neural networks. Our idea capitalized on the recent successful technique of adaptive dropouts combined with the smart data structure (hash tables) based on recently found locality sensitive hashing (LSH) for maximum inner product search (MIPS) \cite{shrivastava2014asymmetric}. We show significant reductions in the computational requirement for training deep networks, without any significant loss in accuracy (within 1\% of the accuracy of the original model on average). In particular, our method achieves the performance of other state-of-the-art regularization methods such as Dropout, Adaptive Dropout, and Winner-Take-All when using only 5\% of the neurons in a standard neural network.
    \item Our proposal reduces computations associated with both the training and testing (inference) of deep networks by reducing the multiplications needed for the feed-forward operation. 
    \item The key idea in our algorithm is to associate LSH hash tables~\cite{gionis1999similarity,Proc:Indyk_STOC98} with every layer. These hash tables support constant-time $O(1)$ insertion and deletion operations.
    \item Our scheme naturally leads to sparse gradient updates. Sparse updates are ideally suited for massively parallelizable asynchronous training \cite{recht2011hogwild}. We demonstrate that this sparsity opens room for truly asynchronous training without any compromise with accuracy. As a result, we obtain near-linear speedup when increasing number of processors.
\end{enumerate}

\section{Related Work}
\label{sec:related_work}
There have been several recent advances aimed at improving the performance of neural networks. \cite{lin2015neural} reduced the number of floating point multiplications by mapping the network's weights stochastically to \{-1, 0, 1\} during forward propagation and performing quantized back-propagation that replaces floating-point multiplications with simple bit-shifts. Reducing precision is an orthogonal approach, and it can be easily integrated with any other approaches. 

\cite{sindhwani2015structured} uses structured matrix transformations with low-rank matrices to reduce the number of parameters for the fully-connected layers of a neural network. This low-rank constraint leads to a smaller memory footprint. However, such an approximation is not well suited for asynchronous and parallel training limiting its scalability. We instead use random but sparse activations, leveraging database advances in approximate query processing, because they can be easily parallelized. [See Section~\ref{sec:lowrank} for more details] 

We briefly review dropouts and its variants, which are popular sparsity promoting techniques relying on sparse activations.  Although such randomized sparse activations have been found favorable for better generalization of deep networks, to the best of our knowledge, this sparsity has not been adequately exploited to make deep networks computationally cheap and parallelizable. We provide first such evidence. 

Dropout \cite{srivastava2014dropout} is primarily a regularization technique that addresses the issue of over-fitting randomly dropping half of the nodes in a hidden layer while training the network. The nodes are independently sampled for every stochastic gradient descent epoch~\cite{srivastava2014dropout}. We can reinterpret Dropout as a technique for reducing the number of multiplications during forward and backward propagation phases, by ignoring nodes randomly in the network which computing the feed-forward pass. It is known that the network's performance becomes worse when too many nodes are dropped from the network. Usually, only 50\% of the nodes in the network are dropped when training the network. At test time, the network takes the average of the thinned networks to form a prediction from the input data, which uses full computations.  

Adaptive dropout \cite{ba2013adaptive} is an evolution of the dropout technique that adaptively chooses the nodes based on their activations. The methodology samples few nodes from the network, where the sampling is done with probability proportional to the node activations which are dependent on the current input.  Adaptive dropouts demonstrate better performance than vanilla dropout~\cite{srivastava2014dropout}. A notable feature of adaptive dropouts was that it was possible to drop significantly more nodes compared to dropouts while still retaining superior performance. Winner-Take-All \cite{makhzani2013k, makhzani2015winner} is an extreme form of adaptive dropouts that uses mini-batch statistics to enforce a sparsity constraint. With this technique, only the k\% largest, non-zero activations are used during the forward and backward phases of training. This approach requires computing the forward propagation step before selecting the k\% nodes with a hard threshold. Unfortunately, all these techniques require full computation of the activations to selectively sample nodes. Therefore, they are only intended for better generalization and not reducing computations.  Our approach uses the insight that selecting a very sparse set of hidden nodes with the highest activations can be reformulated as dynamic approximate query processing problem, which can solve efficiently using locality sensitive hashing. The differentiating factor between adaptive dropout and winner-take-all and our approach is we use sub-linear time randomized hashing to determine the active set of nodes instead computing the inner product for each node individually. 

There is also another orthogonal line of work which uses hashing to reduce memory. \cite{chen2015compressing} introduced a new type of deep learning architecture called Hashed Nets. Their objective was to decrease the number of parameters in the given neural network by using a universal random hash function to tie node weights. Network connections that map to the same hash value (hash collisions) are restricted to have same values. The architecture has the virtual appearance of a regular network while maintaining only a small subset of real weights. We point out that hashed nets are complementary to our approach because we focus on reducing the computational cost involved in using deep nets rather than its size in memory. 

\section{Low-Rank vs Sparsity}
\label{sec:lowrank}
The low-rank (or structured) assumption is very convenient for reducing the complexity of general matrix operations. However, low-rank and dense updates do not promote sparsity and are not friendly for distributed computing. The same principle holds with deep networks. We illustrate it with a simple example. 

Consider a layer of the first network (left) shown in Figure~\ref{fig:structure_matrix}.  The insight is that if the weight matrix $W \in \mathbb{R}^{mxn}$ for a hidden layer has low-rank structure where rank $r \ll min(m,n)$, then it has a representation $W=UV$ where $U \in \mathbb{R}^{mxr}$ and $V \in \mathbb{R}^{rxn}$. This low-rank structure improves the storage requirements and matrix-multiplication time from $O(mn)$ to $O(mr + rn)$. As shown in Figure \ref{fig:structure_matrix}, there is an equivalent representation of the same network using an intermediate hidden layer that contains $r$ nodes and uses the identity activation function. The weight matrices for the hidden layers in the second network (right) map to the matrix decomposition, $U$ and $V$. \cite{sindhwani2015structured} uses structured matrix transformations with low-rank matrices to reduce the number of parameters for the fully-connected layers of a neural network. 

The new network structure and the equivalent structure matrices require a dense gradient update, which is not ideally suited for data parallelism~\cite{recht2011hogwild}. For example, they need sequential gradient descent updates. In this work, we move away from the low-rank assumption and instead make use of randomness and sparsity to reduce computations. We will later show that due to sparsity, our approach is well suited for Asynchronous Stochastic Gradient Descent (ASGD), leading to near-linear scaling.  

It should be noted that our methodology is very different from the notions of making networks sparse by thresholding (or sparsifying) node weights permanently \cite{srivastava2014dropout, makhzani2013k, makhzani2015winner}. Our approach makes use of every node weight but picks them selectively for different inputs. On the contrary, permanent sparsification only has a static set of weights which are used for every input and rest are discarded. 

\begin{figure} [ht]
\begin{center}
  \includegraphics[width=.50\textwidth]{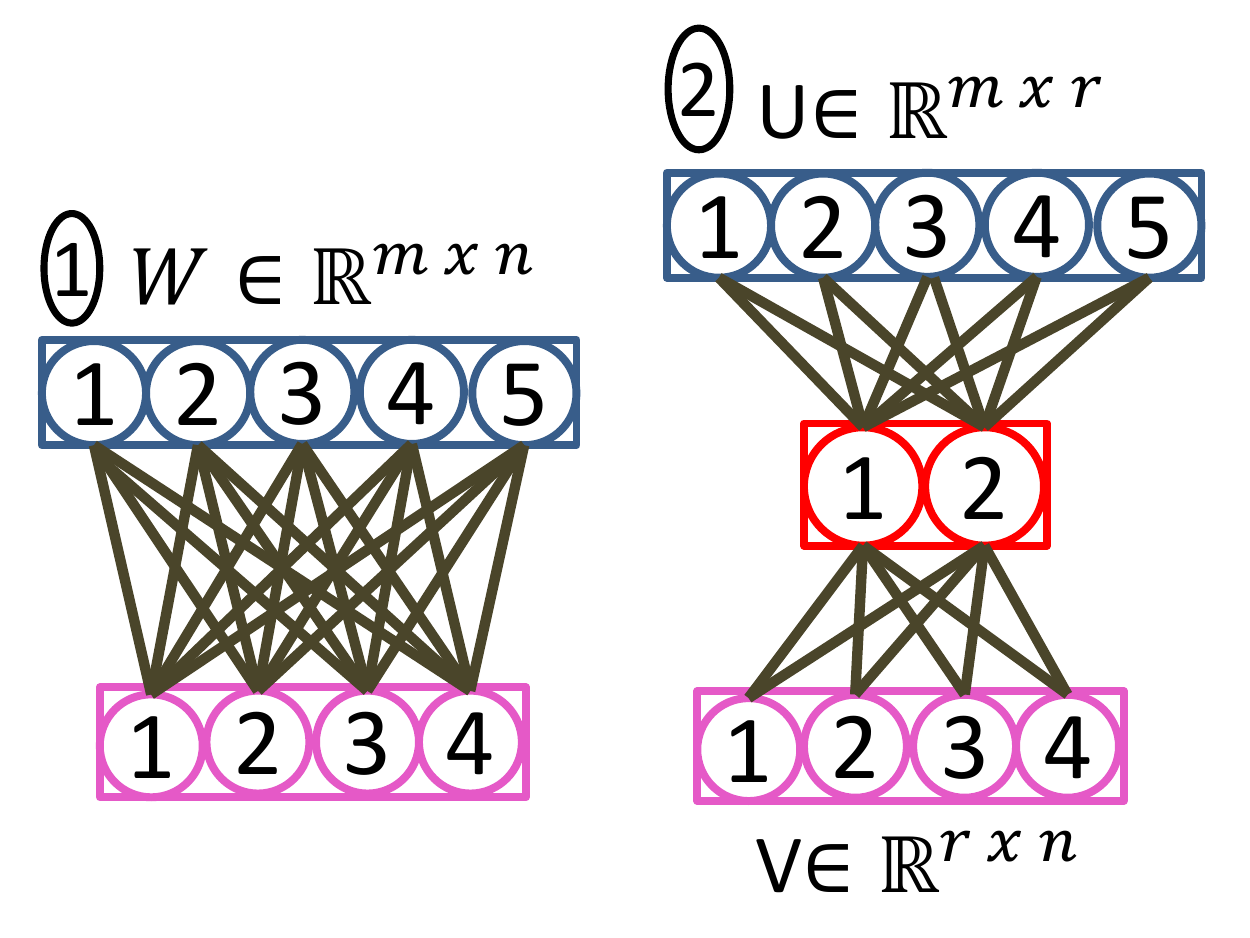}
\end{center}
\begin{caption}
{\bf Illustration of why the Low-Rank assumption for neural networks naturally leads to fewer parameters:} 
(1) This network contains two layers with m and n neurons respectively. The weighted connections between the layers are characterized by the weight matrix $W \in R^{m x n}$ of constrained rank $r$ such that $W = UV$ with $U \in \mathbb{R}^{m x r}$ and $V \in \mathbb{R}^{r x n}$.

(2) An equivalent network contains three layers each with m, r, and n neurons. The top two layers are represented with the matrix $U \in \mathbb{R}^{m x r}$, while the bottom two layers are characterized with the matrix $V \in \mathbb{R}^{r x n}$. The intermediate layer uses the identity activation function $I$ such that the output from network 2 equals that of network 1. $a = f(W^T x) = f((U V)^T x) = f((V^T U^T) x) = f(V^T I(U^T x))$ where $a$ is the layer's output and f is a non-linear activation function.
\end{caption}
\label{fig:structure_matrix}
\end{figure}

\section{Background}\subsection{Artificial Neural Networks}
Neural networks are built from layers of neurons. The combination of multiple hidden layers creates a complex, non-linear function capable of representing any arbitrary function \cite{cybenko1989approximation}. The forward propagation equation is $$a^{l} = f(W^{l} a^{l-1} + b^{l})$$ where $a^{l-1}$ is the output from the previous layer $l-1$, $W^{l}$ is an $[N x M]$ weight matrix for layer $l$, $b^{l}$ is the bias term for layer $l$, and $f$ is the activation function. Common non-linear activation functions are sigmoid, tanh, ReLU. Neural networks are trained using stochastic gradient descent (SGD) - $\theta_{t+1} = \theta_t - \eta \nabla J(\theta)$. The gradient for the neural network $\nabla J(\theta)$ is calculated layer-by-layer using the back-propagation algorithm. The Momentum \cite{polyak1964momentum} technique calculates the parameter update from a weighted sum of the previous momentum term and the current gradient - $\nu_{t+1} = \gamma \nu_t + \eta \nabla J(\theta)$. It reduces the training time of neural networks by mitigating the stochastic nature of SGD. Adaptive gradient descent (Adagrad) \cite{duchi2011adaptive} is a variant of SGD that adapts the learning rate for each parameter. The global learning rate is normalized by the sum of squared gradients. It favors larger learning rates for infrequent or small parameter updates.

\subsection{Adaptive Dropout}
\label{sec:adaptive_dropout}
Adaptive Dropout \cite{ba2013adaptive} is a recently proposed dropout technique that selectively samples nodes with probability proportional to some monotonic function of their activations. To achieve this, authors in \cite{ba2013adaptive} chose a Bernoulli distribution. The idea was that given an input $a_{l-1}$ to a layer $l$, we generate a Bernoulli random variable for every node in the layer with Bernoulli probability $q$ being a monotonic function of the activation. Then, the binary variable sampled from this distribution is used to determine if the associated neuron's activation is kept or set to zero. This adaptive selectivity allows fewer neurons to be activated in the hidden layer and to improve neural network performance over standard Dropout \cite{srivastava2014dropout}. The Winner-Take-All approach \cite{makhzani2015winner} is an extreme case of Adaptive Dropout that takes only the top k\% activations in a hidden layer, which also performs very well.

\subsection{Locality-Sensitive Hashing (LSH)}
Locality-Sensitive Hashing (LSH) \cite{gionis1999similarity, sundaram2013streaming, huang2015query, gao2014dsh, shinde2010similarity} is a popular, sublinear time algorithm for approximate nearest-neighbor search. The high-level idea is to map similar items into the same bucket with high probability. An LSH hash function maps an input data vector to an integer key - $h: x \in \mathbb{R}^D \mapsto [0, 1, 2,\dots, N]$. A collision occurs when the hash values for two data vectors are equal - $h(x) = h(y)$. The probability of collision for an LSH hash function is proportional to the similarity distance between the two data vectors - $Pr[h(x) = h(y)] \propto sim(x,y)$. Essentially, similar items are more likely to collide with the same hash fingerprint and vice versa. 

\subsubsection*{Sub-linear Search with $(K,L)$ LSH Algorithm} To be able to answer approximate nearest-neighbor queries in sub-linear time, the idea is to create hash tables, (see Figure~\ref{fig:algorithm}). Given the collection $\mathcal{C}$, which we are interested in querying for the nearest-neighbor items, the hash tables are generated using the locality sensitive hash (LSH) family.  We assume that we have access to the appropriate locality sensitive hash (LSH) family $\mathcal{H}$ for the similarity of interest. In the classical $(K,L)$ parameterized LSH algorithm, we generate $L$ different hash functions given by  $B_j(x)  = [h_{1j}(x);h_{2j}(x);...;h_{{Kj}}(x)]$. Here $h_{ij}, i \in \{1,2,...,K \}$ and $j \in  \{1,2,...,L \}$, are $KL$ different evaluations of the appropriate locality sensitive hash (LSH) function.  Each of these hash functions is formed by concatenating $K$ sampled hash values from $\mathcal{H}$.

The overall algorithm works in two phases (See~\cite{Report:E2LSH} for details):
\begin{enumerate}
    \itemsep0em
    \item {\bf Pre-processing Phase:} We construct $L$ hash tables from the data by storing all elements $x \in \mathcal{C}$, at location $B_j(x)$ in hash-table $j$ (See Figure~\ref{fig:algorithm} for an illustration). We only store pointers to the vector in the hash tables, as storing data vectors is very memory inefficient.
    \item {\bf Query Phase:} Given a query $Q$, we will search for its nearest-neighbors. We report the union of all the points in the buckets $B_j(Q)$ $\forall j \in \{1,2,...,L\}$, where the union is over $L$ hash tables.  Note, we do not scan all the elements in $\mathcal{C}$, we only probe $L$ different buckets, one from each hash table. 
\end{enumerate}

\cite{Proc:Indyk_STOC98} shows that having an LSH family for a given similarity measure and an appropriate choice of $K$ and $L$, the above algorithm is provably sub-linear. In practice, it is near-constant time because the query only requires few bucket probes.  

\subsubsection*{Hashing Inner Products} It was recently shown that by allowing asymmetric hash functions, the same algorithm can be converted into a sub-linear time algorithm for Maximum Inner Product Search (MIPS) \cite{shrivastava2014asymmetric}. In locality sensitive hashing, self-similarity $Pr[h(x) = h(x)] = 1$ is the closest nearest-neighbor for a query $x$. However, self-similarity does not necessarily correspond with the highest inner product $x^Tx = ||x||^2_2$. Therefore, an asymmetric transformation is necessary to map maximum inner product search (MIPS) to the standard nearest-neighbor search (LSH). Our work heavily leverages these algorithms. For the purpose of this paper, we will assume that we can create a very efficient data structure, which can be queried for large inner products. See~\cite{shrivastava2014asymmetric} for details.

\subsubsection*{Multi-Probe LSH} One common complaint with classical LSH algorithm is that it requires a significant number of hash tables. Large $L$ increases the hashing time and memory cost. A simple solution was to probe multiple "close-by" buckets in every hash table rather than probing only one bucket~\cite{lv2007multi}. Thus, for a given query $Q$, in addition to probing  $B_j(Q)$ in hash table $j$, we also generate several new addresses to probe by slightly perturbing values of $B_j(Q)$. This simple idea significantly reduces the number of tables needed with LSH, and we can work with only a few hash tables. See~\cite{lv2007multi} for more details. 

\section{Proposed Methodology}
\subsection{Intuition}
Winner-Take-All~\cite{makhzani2013k, makhzani2015winner} technique shows that we should only consider a few nodes (top k\%) with large activations for a given input and ignore the rest while computing the feed-forward pass. Furthermore, the back-propagation updates should only be performed on those few chosen weights. Let us denote $AS$ (Active Set) to be the top k\% nodes having significant activations. Let $n$ denote the total number of nodes in the neural network with $AS \ll n$. For one gradient update, Winner-Take-All needs to perform first $O(n \log{} n)$ work to compute $AS$, followed by updating $O(AS)$ weights.  $O(n log n)$ seems quite wasteful. In particular, given the input, finding the active set $AS$ is a search problem, which can be solved efficiently using smart data structures. Furthermore, if the data structure is dynamic and efficient, then the gradient updates will also be efficient. 

For a node $i$ with weight $w_i$ and input $x$, its activation is a monotonic function of the inner product $w_i^Tx$. Thus, given $x$, selecting nodes with large activations (the set $AS$) is equivalent to searching $w_i$, from a collection of weight vectors, have large inner products with $x$. Equivalently from query processing perspective, if we treat input $x$ as a query, then the search problem of selecting top k\% nodes can be done efficiently, in sub-linear time in the number of nodes, using the recent advances in maximum inner product search \cite{shrivastava2014asymmetric}.  Our proposal is to create hash tables with indexes generated by asymmetric locality sensitive hash functions tailored for inner products. With such hash tables, given the query input $x$, we can very efficiently approximate the active set $AS$. 

There is one more implementation challenge. We also have to update the nodes (weights associated with them) in $AS$ during the gradient update. If we can perform these updates in $O(AS)$ instead of $O(n)$, then we save significant computation. Thus, we need a data structure where updates are efficient as well. With years of research in data structures, there are many choices which can make update efficient. We describe our system in details in Section~\ref{sec:system}.  

\subsection{Equivalence with Adaptive Dropouts}
It turns out that using randomized asymmetric locality sensitive hashing \cite{shrivastava2014asymmetric} for finding nodes with large inner products is formally, in a statistical sense, equivalent to adaptive dropouts with a non-trivial sampling distribution. 

Adaptive Dropout technique uses the Bernoulli distribution (Section~\ref{sec:adaptive_dropout}) to adaptively sample nodes with large activation. In theory, any distribution assigning probabilities in proportion to nodes activation is sufficient. We argue that Bernoulli is a sub-optimal choice. There is another non-intuitive, but a very efficient choice. This choice comes from the theory of Locality-Sensitive Hashing (LSH) \cite{Proc:Indyk_STOC98}, which is primarily considered a black box technique for fast sub-linear search. Our core idea comes from the observation that, for a given search query $q$, the $(K,L)$ parametrized LSH algorithm \cite{Report:E2LSH} inherently samples, in sub-linear time, points from datasets with probability proportional to $1 - (1-p^K)^L$. Here $p$, is the collision probability, a monotonic function of the similarity between the query and the retrieved point. See \cite{Report:E2LSH} for details.  The expression holds for any $K$ and $L$. $K$ controls the sparsity of buckets. With a reasonable choice of $K$, we can always make buckets arbitrarily sparse. Even if the buckets are heavy, we can always do a simple random sampling, in the buckets, and adding only a constant factor in the probability $p$, keeping the monotonicity intact. 

\begin{theorem} 
 {\bf Hash-Based Sampling} - For a given input $x$ to any layer, any $(K,L)$ parametrized hashing algorithm selects a node $i$, associated with weight vector $w_i$, with probability proportional to  $1-(1-p^K)^L$. Here $p = Pr(h(I) = h(w_i))$ is the collision probability of the associated locality sensitive hash function. The function $1- (1-p^K)^L$ is monotonic in $p$.
\end{theorem}

\begin{proof}
The probability of finding node $i$ in any hash table is $p^K$ (all $K$ bits collide where each bit occurs with probability $p$), so the probability of missing node $i$ in $L$ independent hash tables is $(1-p^K)^L$. If we further choose to sample the buckets randomly by another constant factor $r < 1$, then effective collision probability becomes $p' = rp$.
\end{proof}

With recent advances in hashing inner products, we can make $p$ a monotonic function of inner products \cite{shrivastava2014asymmetric, shrivastava2014improved, shrivastava2015asymmetric,shrivastava2015probabilistic}. Translating into deep network notation, given the activation vector of previous layer $a^{l-1}$ as the query and layer's weights $W^{l}$ as the search space, we can sample with probability proportional to $$g(W^{l}) = 1-(1-f(W^{l} a^{l-1} )^K)^L$$ which is a monotonic function of node activation. Thus, we can naturally integrate adaptive dropouts with hashing. Overall, we are simply choosing a specific form of adaptive dropouts, which leads to efficient sampling in time sub-linear in the number of neurons (or nodes). 

\begin{corollary}
{\bf Adaptive Dropouts are Efficient} - There exists an efficient family of sampling distributions for adaptive dropouts, where the sampling and updating cost is sub-linear in the number of nodes. In particular, we can construct an efficient sampling scheme such that for any two nodes $i$ and $j$, in any given layer $l$, the probability of choosing node $i$ is greater than the probability of choosing node $j$ if and only if the current activation of node $i$ is higher than that of $j$.
\end{corollary}

\begin{proof}
Since asymmetric LSH functions \cite{shrivastava2014asymmetric, shrivastava2014improved, shrivastava2015asymmetric} have a collision probability $p$ that is a monotonic function of the inner product, it is also a montonic function of the neuron's activation. All activation functions including sigmoid \cite{funahashi1989approximate}, tanh \cite{lecun2012efficient}, and ReLU \cite{glorot2011deep} are monotonic functions of inner products. Composing monotonic functions retains monotonicity. Therefore, the corollary follows because $1-(1-p^K)^L$ is monotonic with respect to $p$. In addition, monotonicity goes in both direction as the inverse function exists and it is also monotonic.
\end{proof}

\begin{figure}[ht]
\mbox{
  \includegraphics[width=.50\textwidth]{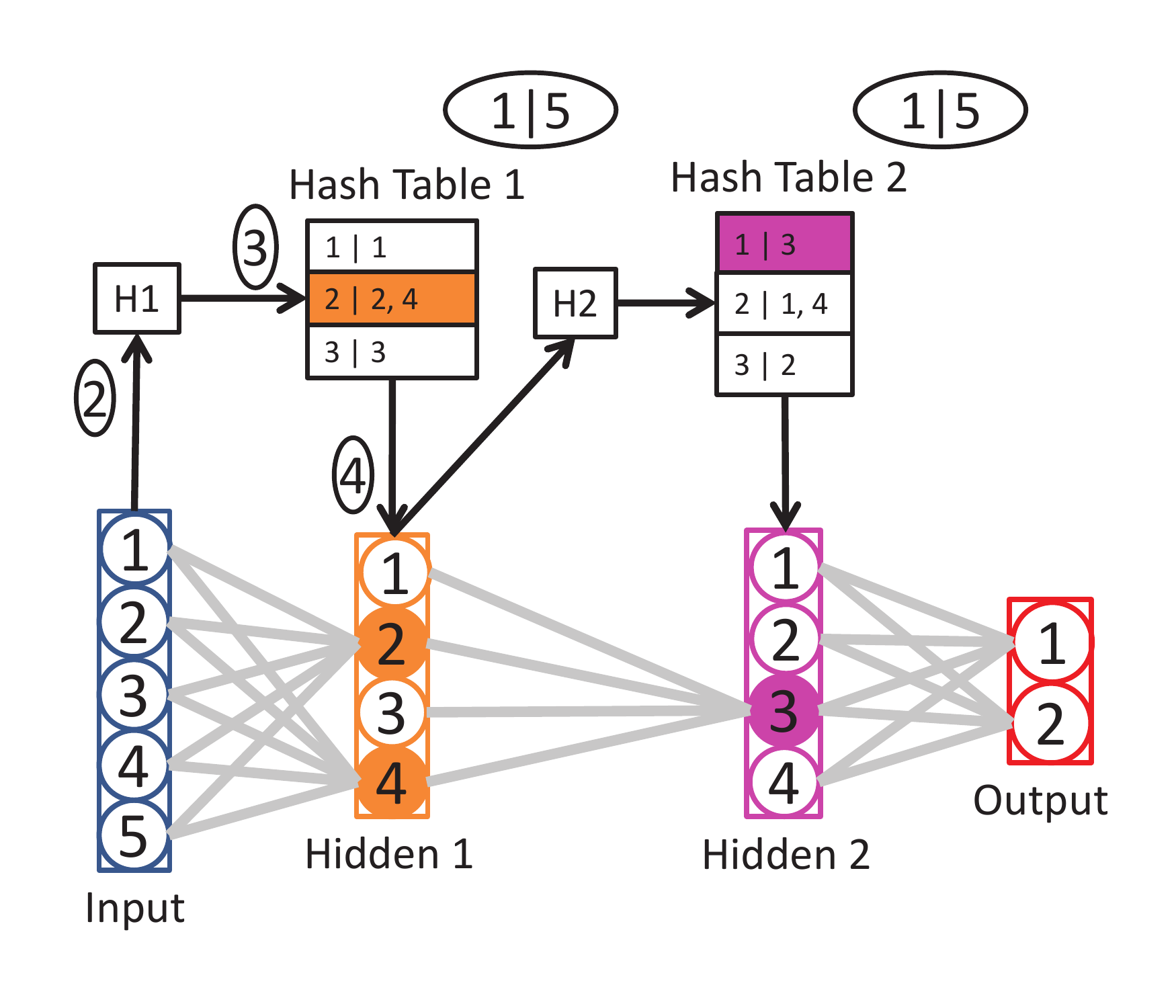}} \vspace{-0.3in}
\caption{{\bf A visual representation of a neural network using randomized hashing}
(1) Build the hash tables by hashing the weights of each hidden layer (one time operation)
(2) Hash the layer's input using the layer's randomized hash function
(3) Query the layer's hash table(s) for the active set $AS$
(4) Only perform forward and back-propagation on the neurons in the active set. The solid-colored neurons in the hidden layer are the active neurons.
(5) Update the $AS$ weights and the hash tables by rehashing the updated weights to new hash locations. }
  \label{fig:algorithm}
\end{figure}

\subsection{Hashing Based Back-Propagation}
\label{sec:system}
As argued, our approach is simple. We use randomized hash functions to build hash tables from the nodes in each hidden layer. We sample nodes from the hash table with probability proportional to the node's activation in sub-linear time. We then perform forward and back propagation only on the active nodes retrieved from the hash tables. We later update the hash tables to reorganize only the modified weights. 

Figure \ref{fig:algorithm} illustrates an example neural network with two hidden layers, five input nodes, and two output nodes. Hash tables are built for each hidden layer, where the weighted connections for each node are hashed to place the node in its corresponding bucket. Creating hash tables to store all the initial parameters is a one-time operation which requires cost linear in the number of parameters. 

During a forward propagation pass, the input to the hidden layer is hashed with the same hash function used to build the hidden layer's hash table. The input's fingerprint is used to collect the active set $AS$ nodes from the hash table. The hash table only contains pointers to the nodes in the hidden layer. Then, forward propagation is performed only on the nodes in the active set $AS$. The rest of the hidden layer's nodes, which are not part of the active set, are ignored and are automatically switched off without even touching them. On the back propagation pass, the active set is reused to determine the gradient and to update the parameters. We rehash the nodes in each hidden layer to account for the changes in the network during training. 

In detail, the hash function for each hidden layer is composed of $K$ randomized hash functions. We use the sign of an asymmetrically transformed random projection (see~\cite{shrivastava2014improved} for details) to generate the $K$ bits for each data vector. The $K$ bits are stored together efficiently as an integer, forming a fingerprint for the data vector. We create a hash table of $2^K$ buckets, but we only keep the nonempty buckets to minimize the memory footprint (analogous to hash-maps in Java). Each bucket stores pointers to the nodes whose fingerprints match the bucket's id instead of the node itself. In figure \ref{fig:algorithm}, we showed only one hash table, which is likely to miss valuable nodes in practice. In our implementation, we generate $L$ hash tables for each hidden layer, and each hash table has an independent set of $K$ random projections. Our final active set from these $L$ hash tables is the union of the buckets selected from each hash table. For each layer, we have $L$ hash tables. Effectively, we have two tunable parameters, $K$ bits and $L$ tables to control the size and the quality of the active sets. The $K$ bits increase the precision of the fingerprint, meaning the nodes in a bucket are more likely to generate higher activation values for a given input. The $L$ tables increase the probability of finding useful nodes that are missed because of the randomness in the hash functions.

\begin{algorithm}[tb]
   \caption{Deep Learning with Randomized Hashing}
   \label{alg:algorithm}
\begin{algorithmic}
   \STATE // HF$_l$ - Layer $l$ Hash Function
   \STATE // HT$_l$ - Layer $l$ Hash Tables
   \STATE // AS$_l$ - Layer $l$ Active Set
   \STATE // $\theta^{l}_{AS} \in W^{l}_{AS}$, $b^{l}_{AS}$ - Layer $l$ Active Set parameters
   \STATE Randomly initialize parameters $W^l$, $b^l$ for each layer $l$
   \STATE HF$_l$ = constructHashFunction($k$, $L$)
   \STATE HT$_l$ = constructHashTable($W^l$, HF$_l$)
   \WHILE {not stopping criteria}
   \FOR {{\bf each} training epoch}
   \STATE // Forward Propagation
   \FOR {layer $l = 1 \dots N$}
   \STATE fingerprint$_l$ = HF$_l(a_l)$
   \STATE AS$_l$ = collectActiveSet(HT$_l$, fingerprint$_l$)
   \FOR {{\bf each} node $i$ in AS$_l$}
   \STATE $a^{l+1}_{i} = f(W^{l}_{i}a^{l}_{i} + b^{l}_{i})$
   \ENDFOR
   \ENDFOR
   \STATE // Backpropagation
   \FOR {layer $l = 1 \dots N$}
   \STATE $\Delta J(\theta^{l}_{AS})$ = computeGradient($\theta^{l}_{AS}$, AS$_l$)
   \STATE $\theta^{l}_{AS}$ = updateParameters($\theta^{l}_{AS}$, $\Delta J(\theta_{AS}$))
   \ENDFOR
   \STATE {\bf for each} Layer $l$ -> updateHashTables(HF$_l$, HT$_l$, $\theta^{l}$)
   \ENDFOR
   \ENDWHILE
\end{algorithmic}
\end{algorithm}
\subsection{Efficient Query and Updates}

Our algorithm critically depends on the efficiency of the query and update procedure. The hash table is one of the most efficient data structures, so this is not a difficult challenge. Querying a single hash table is a constant time operation when the bucket size is small. Bucket size can be controlled by $K$ and by sub-sampling the bucket. There is always a possibility of crowded buckets due to bad randomness or because of too many near-duplicates in the data. These crowded buckets are not very informative and can be safely ignored or sub-sampled. 

We never need to store weights in the hash tables. Instead, we only store references to the weight vectors, which makes hash tables a very light entity. Further, we reduce the number of different hash tables required by using multi-probe LSH \cite{lv2007multi}.

Updating a weight vector $w_i$ associated with a node $i$ can require changing the location of $w_i$ in the hash table, if the updates lead to a change in the hash value. With the hash table data structure, the update is not an issue when the buckets are sparse.  Updating a weight only requires one insertion and one deletion in the respective buckets. There are plenty of choices for efficient insert and delete data structure for buckets. In theory, even if buckets are not sparse, we can use a red-black-tree \cite{cormen2009introduction} to ensure both insertion and deletion cost is logarithmic in the size of the bucket. However, using simple arrays when the buckets are sparse is preferable because they are easy to parallelize. With arrays insertion is O(1) and deletion is $O(b)$, where $b$ is the size of buckets. Controlling the size of $b$ can be easily achieved by sub-sampling the bucket. Since we create $L$ independent hash tables, even for reasonably large $L$, the process is quite robust to cheap approximations. We further reduce the size of $L$ using multi-probing \cite{lv2007multi}. Multi-probe with binary hash function is quite straightforward. We just have to randomly flip few bits of the $K$-bit hash to generate more addresses. 

There are plenty of other choices than can make hashing significantly faster, such as cheap re-ranking~\cite{shrivastava2012fast}. See~\cite{Shrivastava:SOCC_16}, where authors show around 500x reduction in computations for image search by incorporating different algorithmic and systems choices. 

\subsection{Overall Cost}
In every layer, during every Stochastic Gradient Descent (SGD) update, we compute $K \times L$ hashes of the input, probe around $10L$ buckets, and take their union. In our experiments, we use $K=6$ and $L=5$, i.e. 30 hash computations only. There are many techniques to further reduce this hashing cost~\cite{achlioptas2001database,li2006very,Proc:OneHashLSH_ICML14,Proc:Shrivastava_UAI14,Shrivastava:NIPS_16}. We probe around 10 buckets in each hash tables for obtaining 5\% of active nodes, leading to the union of 50 buckets in total. The process gives us the active set $AS$ of nodes, which is usually significantly smaller when compared to the number of nodes $n$. During SGD, we update all of the weights in $AS$ along with the hash table. Overall, the cost is of the order of the number of nodes in $AS$, which in our experiments show can go as low as $5\%$. For 1000 node in the layer, we have to update around 10-50 nodes only. The bottleneck cost is calculations of activations (actual inner products) of these nodes in the $AS$ which for every node is around $1000$ floating point multiplications each. The benefits will be even more significant for larger networks. 

\subsection{Bonus: Sparse Updates can be Parallelized}
As mentioned, we only need to update the set of weights associated with nodes in the active set $AS$. If the $AS$ is very sparse, then it is unlikely that multiple SGD updates will overwrite the same set of weights. Intuitively, assuming enough randomness in the data vector, we will have a small active set $AS$ chosen randomly from among all the nodes. It is unlikely that multiple active sets of randomly selected data vectors will have significant overlaps. Small overlaps imply fewer conflicts while updating. Fewer conflicts while updating is an ideal ground where SGD updates can be parallelized without any overhead. In fact, it was both theoretically and experimentally that random and sparse SGD updates can be parallelized without compromising with the convergence ~\cite{recht2011hogwild}. Parallel SGD updates is one the pressing challenges in large-scale deep learning systems~\cite{chen2016revisiting}. Vanilla SGD for deep networks is sequential, and parallel updates lead to poor convergence due to significant overwrites. Our experimental results, in Section~\ref{sec:ASGD}, support these known phenomena. Exploiting this unique property,  we show near linear scaling of our algorithm with increasing processors without hurting convergence. 

\section{Evaluations}
We design experiments to answer the following six important questions: 
\begin{enumerate}
    \item How much computation can we reduce without affecting the vanilla network's accuracy? 
    \item How effective is adaptive sampling compared to a random sampling of nodes?
    \item How does approximate hashing compare with expensive but exact approaches of adaptive dropouts~\cite{ba2013adaptive} and Winner-Takes-all~\cite{makhzani2013k, makhzani2015winner} in terms of accuracy? In particular, is hashing working as intended? 
    \item How is the network's convergence affected by increasing number of processors as we perform parallel SGD updates using our approach?
    \item Is sparse update necessary? Is there any deterioration in performance, if we perform vanilla dense updates in parallel?
    \item  What is the wall clock decrease in training time, as a function of increasing number of processors?
\end{enumerate}

For evaluation, we implemented the following five approaches to compare and contrast against.
\begin{itemize}
    \item Standard (NN) : A full-connected neural network
    \item Dropout (VD) \cite{srivastava2014dropout}: A neural network that disables the nodes of a hidden layer using a fixed probability threshold
    \item Adaptive Dropout (AD) \cite{ba2013adaptive}: A neural network that disables the nodes of a hidden layer using a probability threshold based on the inner product of the node’s weights and the input.
    \item Winner Take All (WTA) \cite{makhzani2013k, makhzani2015winner}: A neural network that sorts the activations of a hidden layer and selects the k\% largest activations
    \item Randomized Hashing (LSH): A neural network that selects nodes using randomized hashing. A hard threshold limits the active node set to k\% sparsity
\end{itemize}

\subsection{Datasets}
To test our neural network implementation, we used four publicly available datasets - MNIST8M \cite{loosli-canu-bottou-2006}, NORB \cite{lecun2004learning}, CONVEX, and RECTANGLES \cite{larochelle2007empirical}. The statistics of these datasets are summarized in Table~\ref{fig:dataset_size}. The MNIST8M, CONVEX, and RECTANGLES datasets contain 28x28 images, forming 784-dimensional feature vectors. The MNIST8M task is to classify each handwritten digit in the image correctly. It is derived by applying random deformations and translations to the MNIST dataset. The CONVEX dataset objective is to identify if a single convex region exists in the image. The goal for the RECTANGLES dataset is to discriminate between tall and wide rectangles overlaid on a black and white background image. The NORB dataset \cite{lecun2004learning} contains images of 50 toys, belonging to 5 categories under various lighting conditions and camera angles. Each data point is a 96x96 stereo image pair. We resize the image from 96x96 to 32x32 and concatenate the image pairs together to form a 2048-dimensional feature vector.

\begin{figure}
    \begin{center}
    \begin{tabular}{ |c|c|c| } 
    \hline
    Dataset & Train Size & Test Size \\ 
    \hline
    MNIST8M & 8,100,000 & 10,000 \\ 
    NORB & 24,300 & 24,300 \\
    Convex & 8,000 & 50,000 \\
    Rectangles & 12,000 & 50,000 \\
    \hline
    \end{tabular}
    \end{center}
    \caption{Dataset - Training + Test Size}
    \label{fig:dataset_size}
\end{figure}

\begin{figure*}[ht]
\begin{center}
\mbox{\hspace{-0.185in}
\includegraphics[width=1.9in]{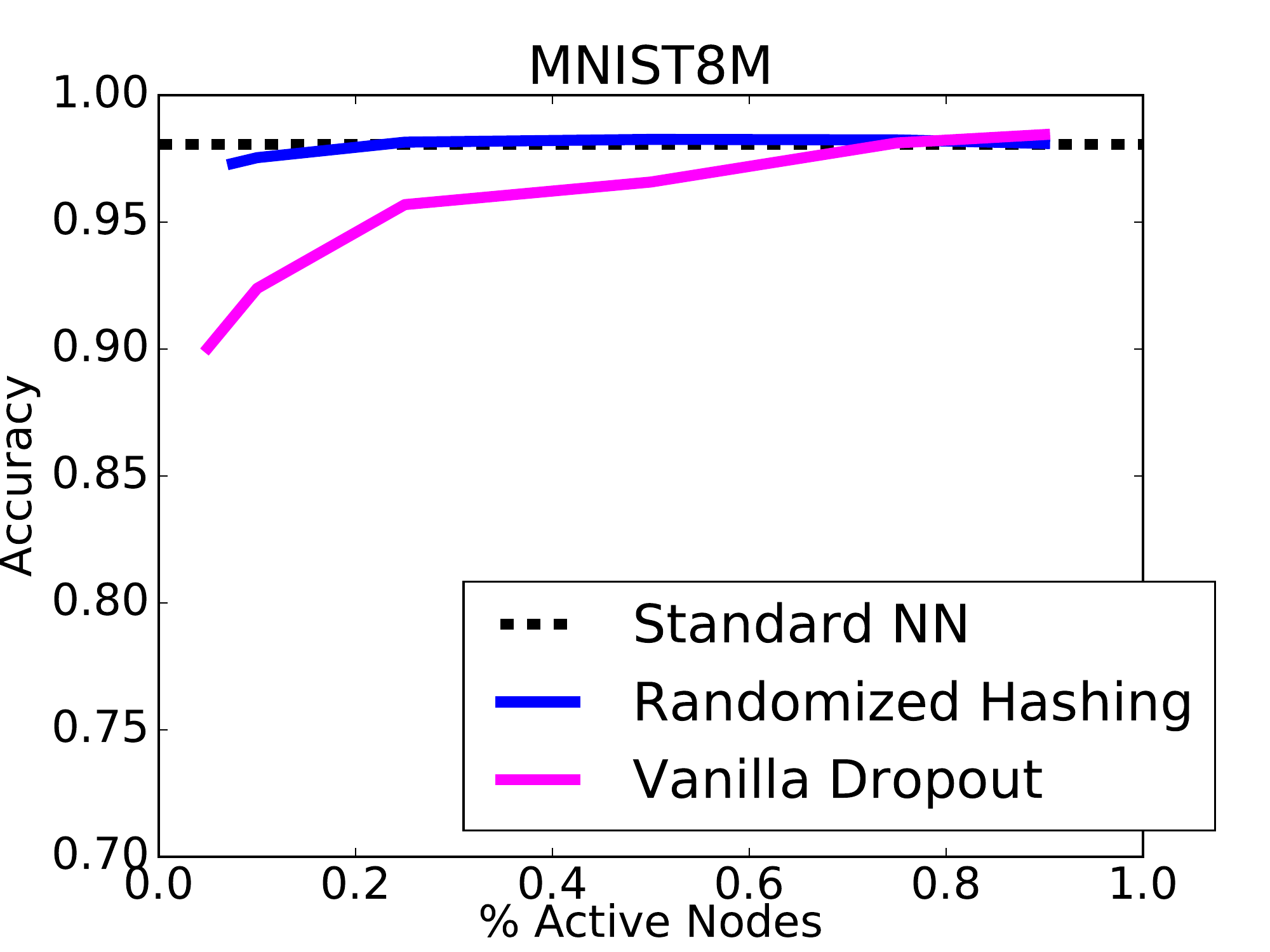} \hspace{-0.15in}
\includegraphics[width=1.9in]{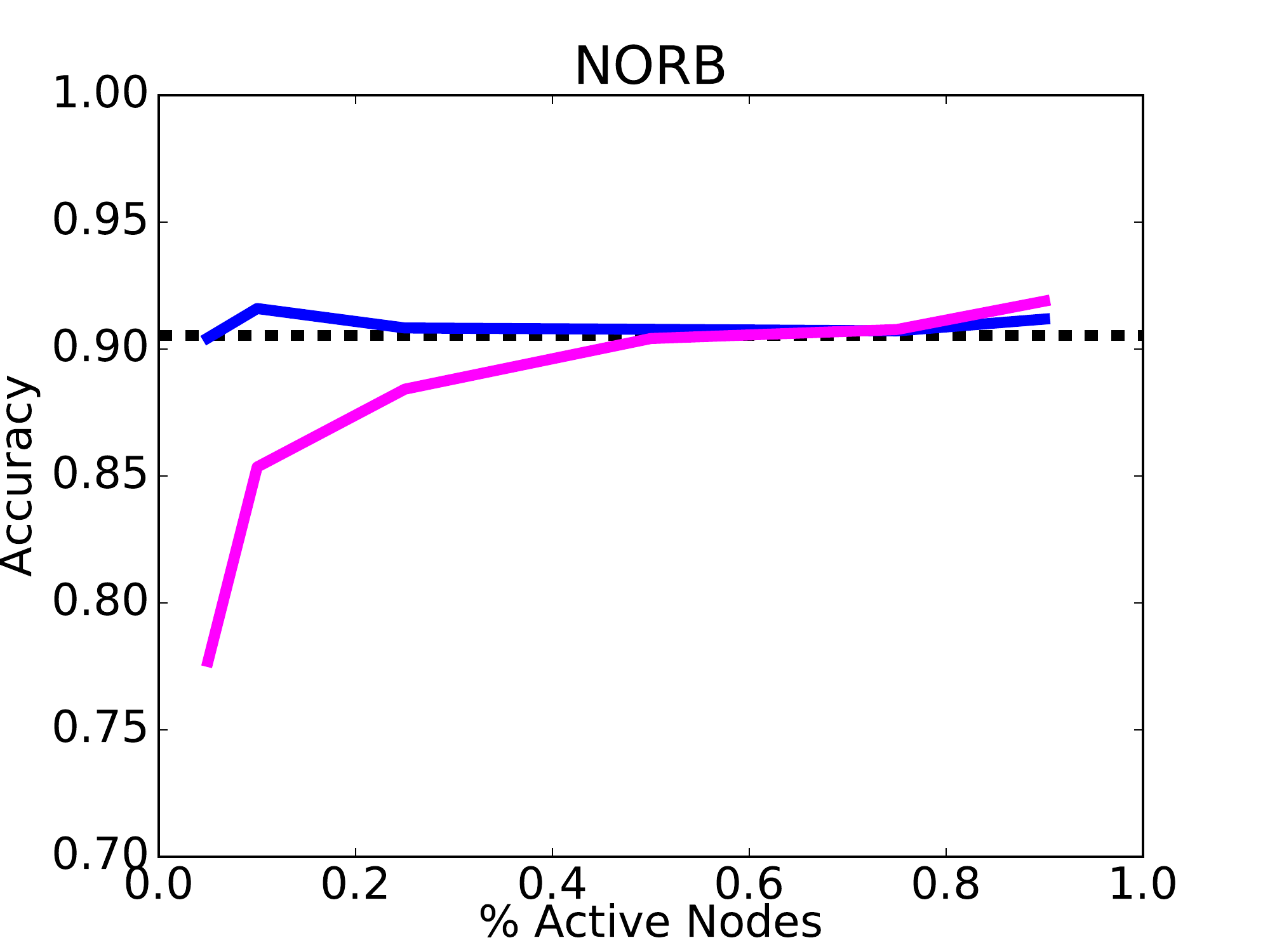} \hspace{-0.19in}
\includegraphics[width=1.9in]{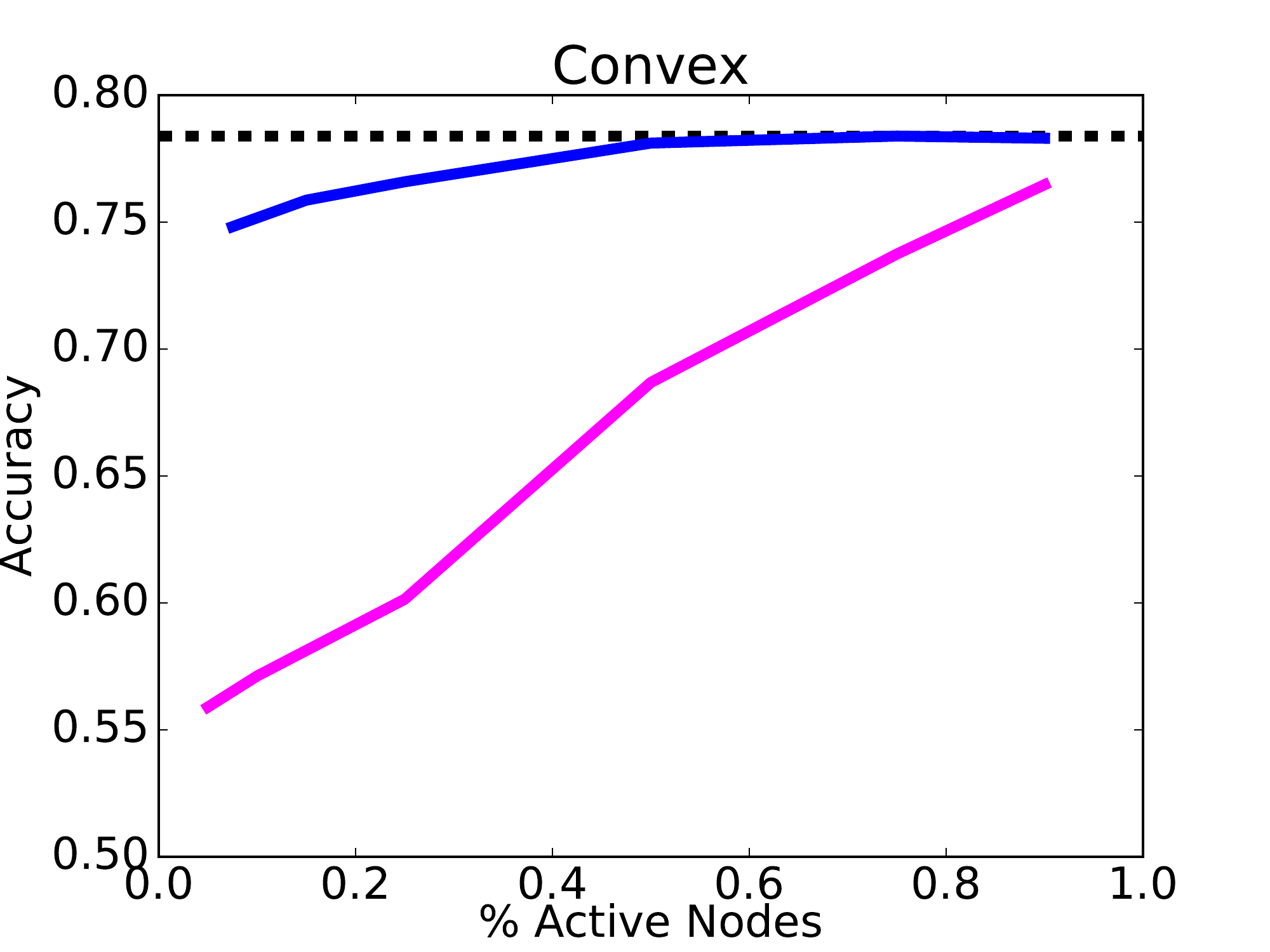} \hspace{-0.19in}
\includegraphics[width=1.9in]{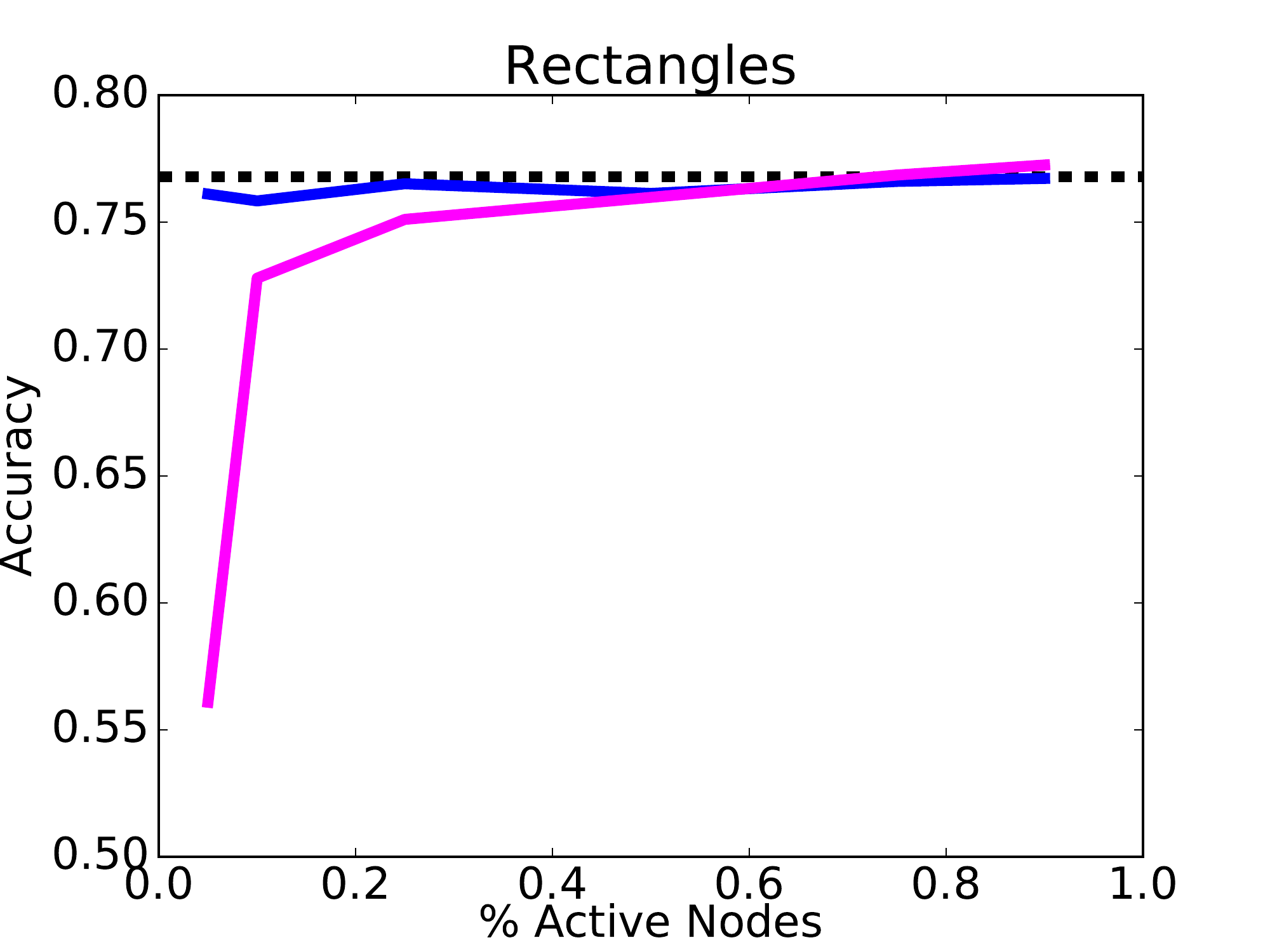} \hspace{-0.19in}}
\mbox{\hspace{-0.185in}
\includegraphics[width=1.9in]{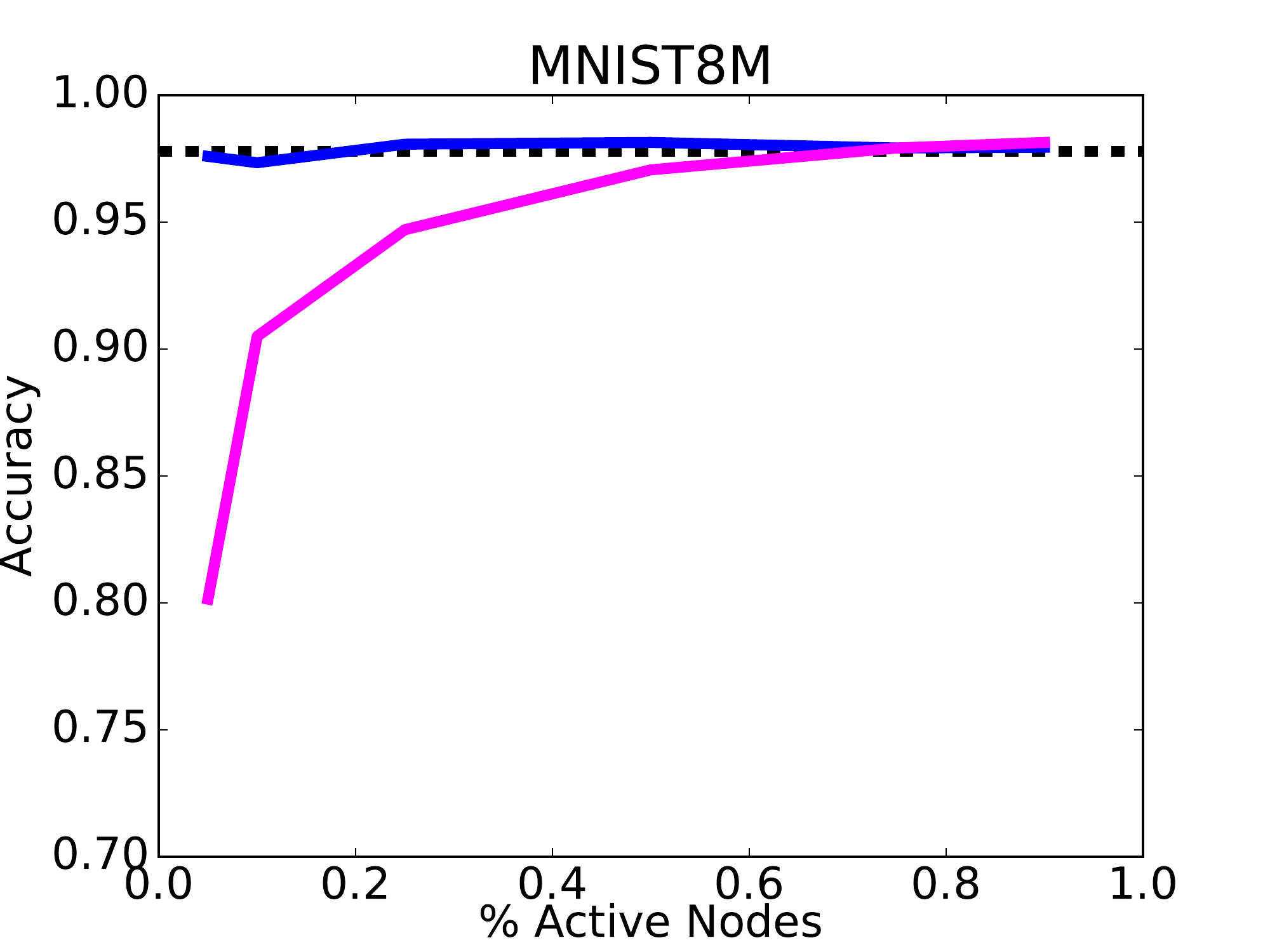} \hspace{-0.15in}
\includegraphics[width=1.9in]{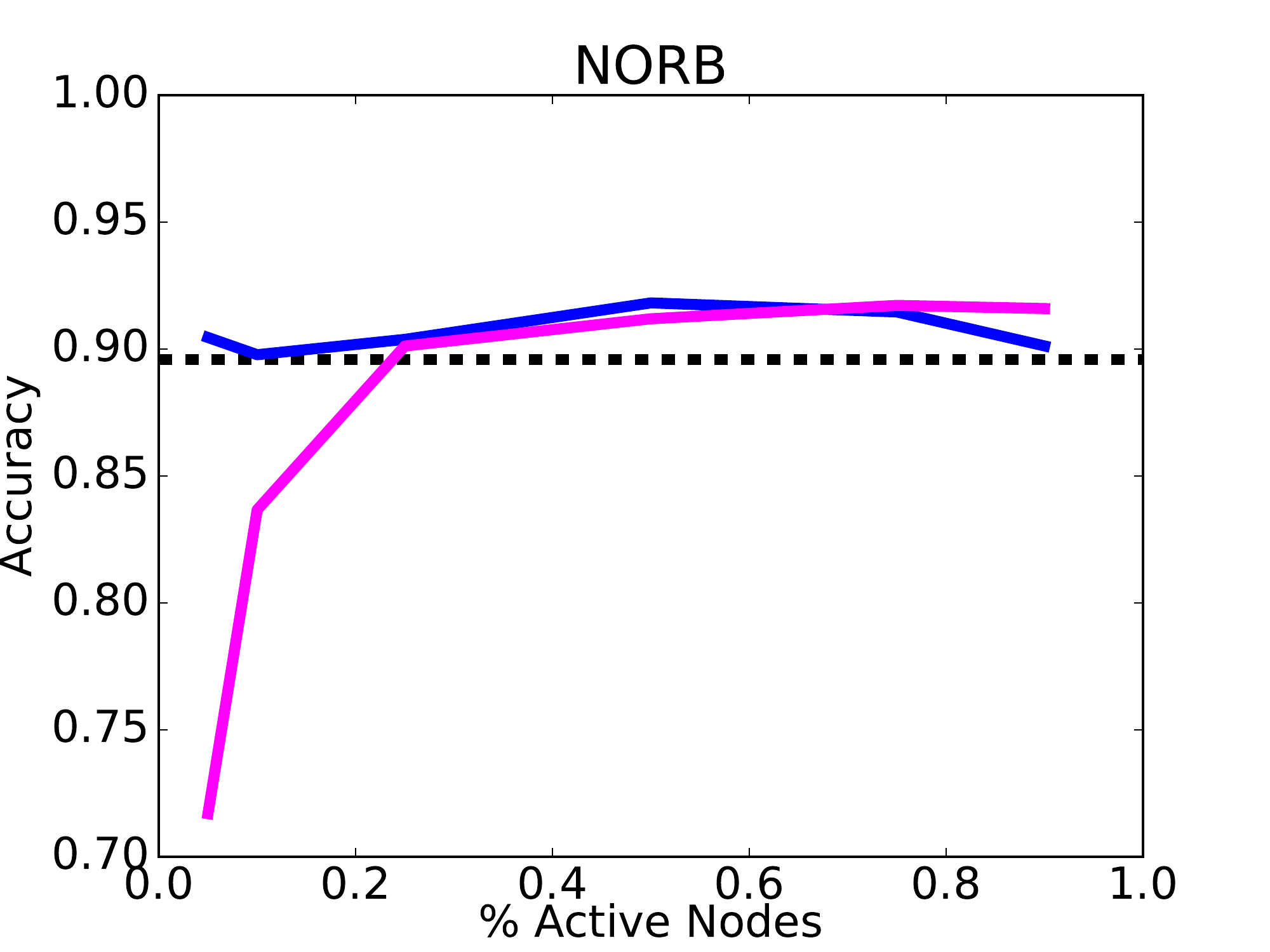} \hspace{-0.19in}
\includegraphics[width=1.9in]{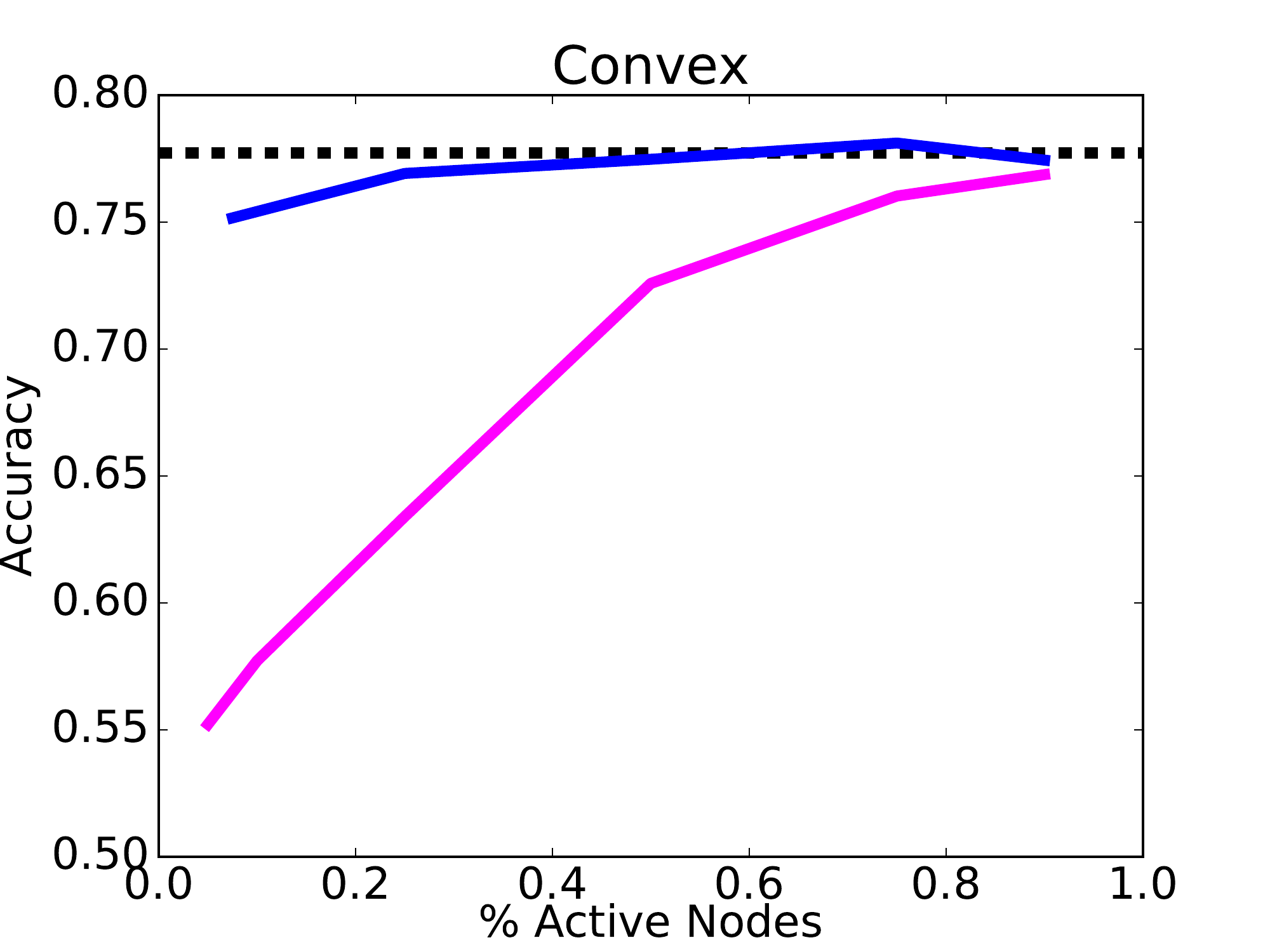} \hspace{-0.19in}
\includegraphics[width=1.9in]{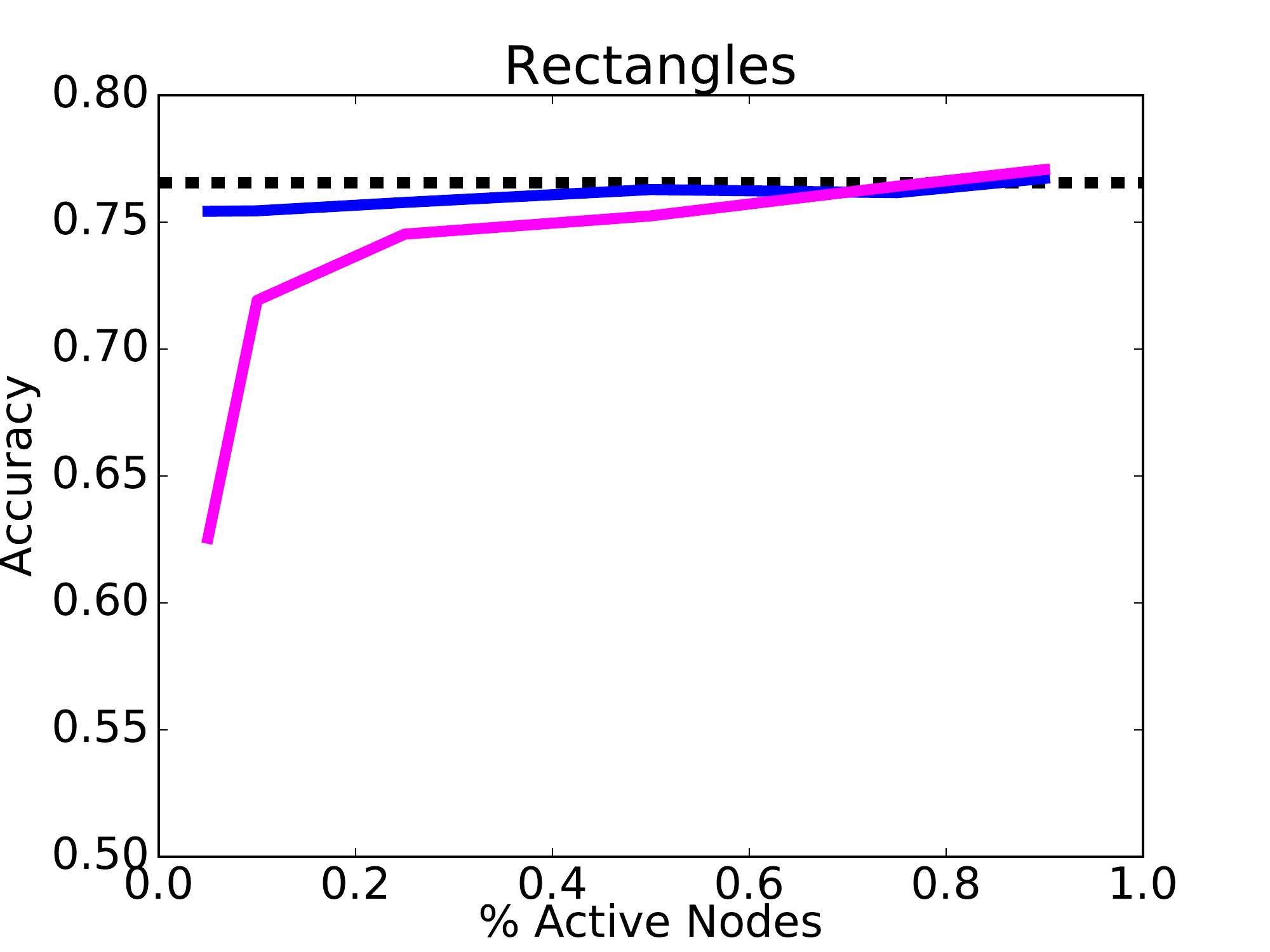} \hspace{-0.19in}}
\end{center}
\caption{Classification accuracy under different levels of active nodes with networks on the MNIST (1st), NORB (2nd), Convex (3rd) and Rectangles (4th) datasets. The standard neural network (dashed black line) is our baseline accuracy. We can clearly see that adaptive sampling with hashing (LSH) is significantly more effective than random sampling (VD). {\bf 1. Top Panels:} 2 hidden Layers. {\bf 1. Bottom Panels:} 3 hidden Layers}
  \label{LSH_VD_STD}
\end{figure*}

\begin{figure*}[ht]
\begin{center}
\mbox{\hspace{-0.185in}
\includegraphics[width=1.9in]{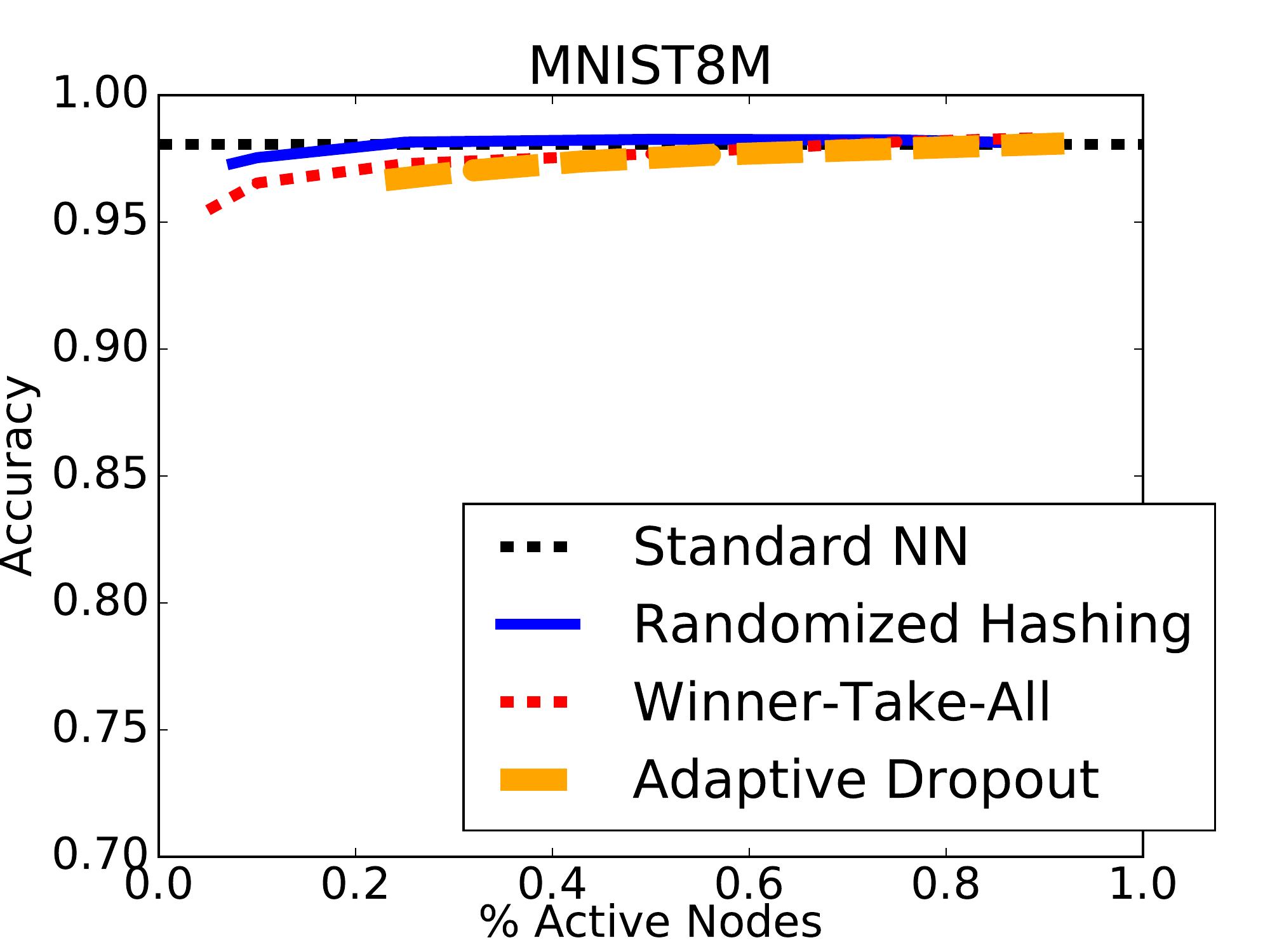} \hspace{-0.15in}
\includegraphics[width=1.9in]{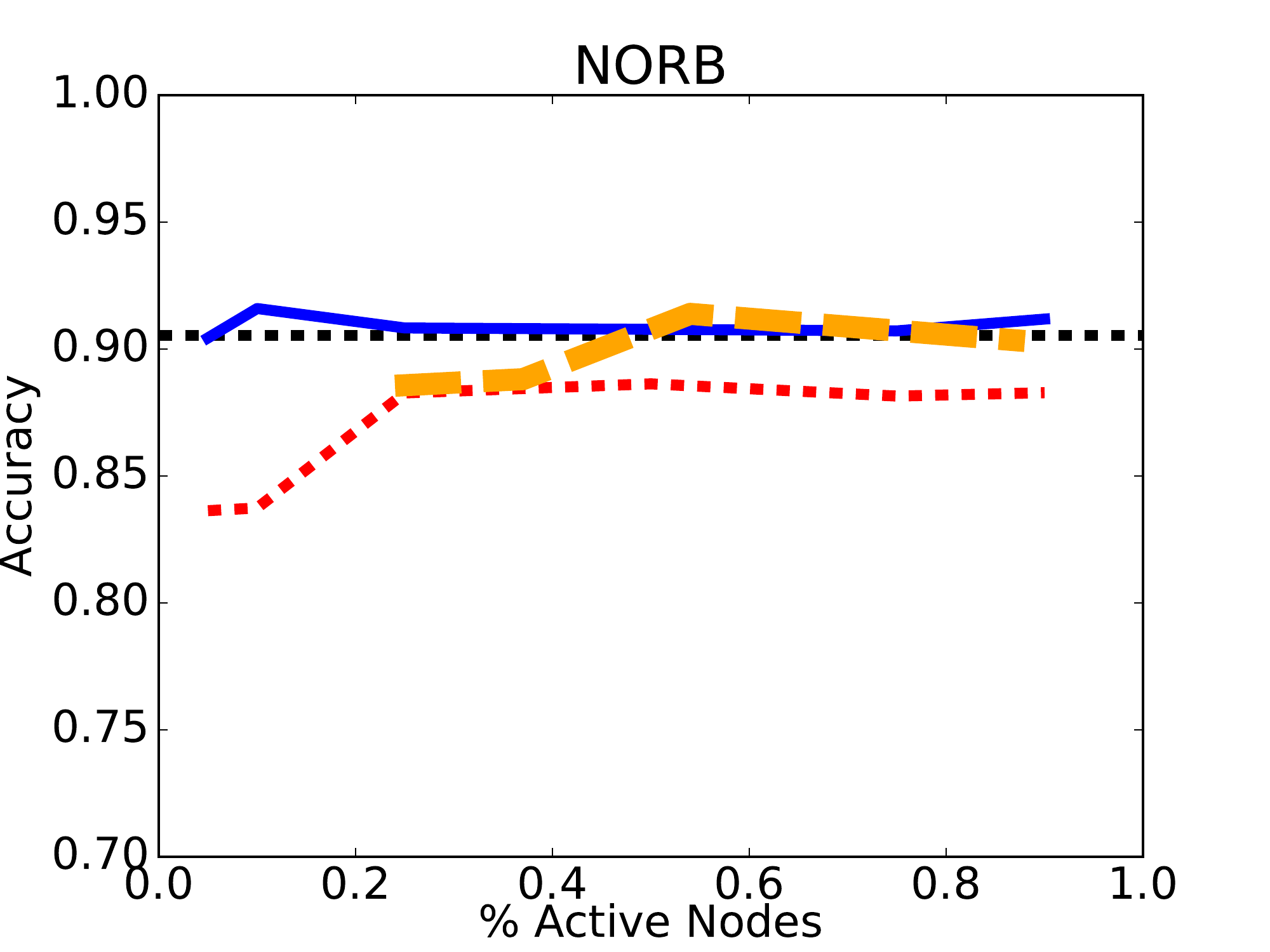} \hspace{-0.19in}
\includegraphics[width=1.9in]{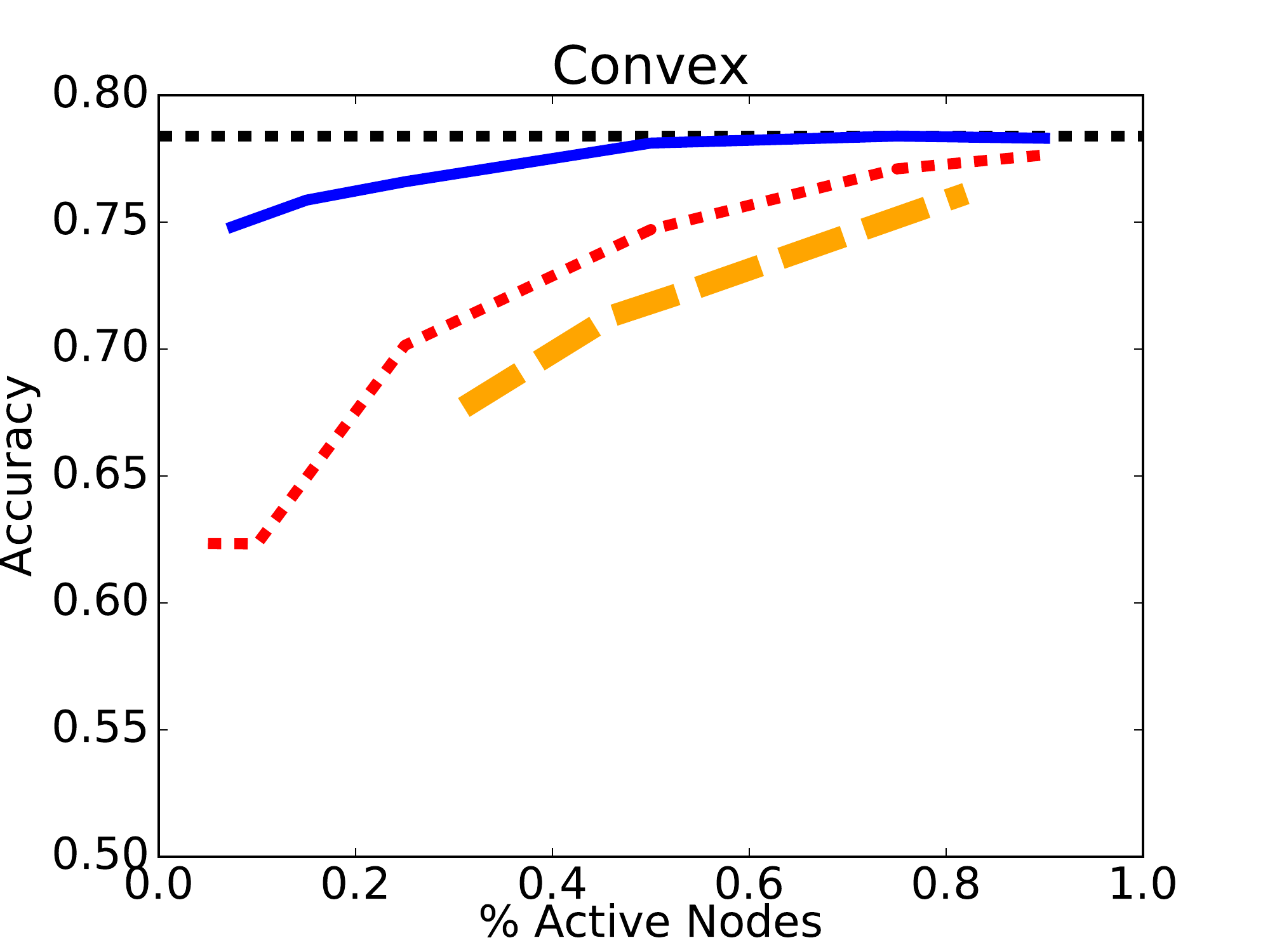} \hspace{-0.19in}
\includegraphics[width=1.9in]{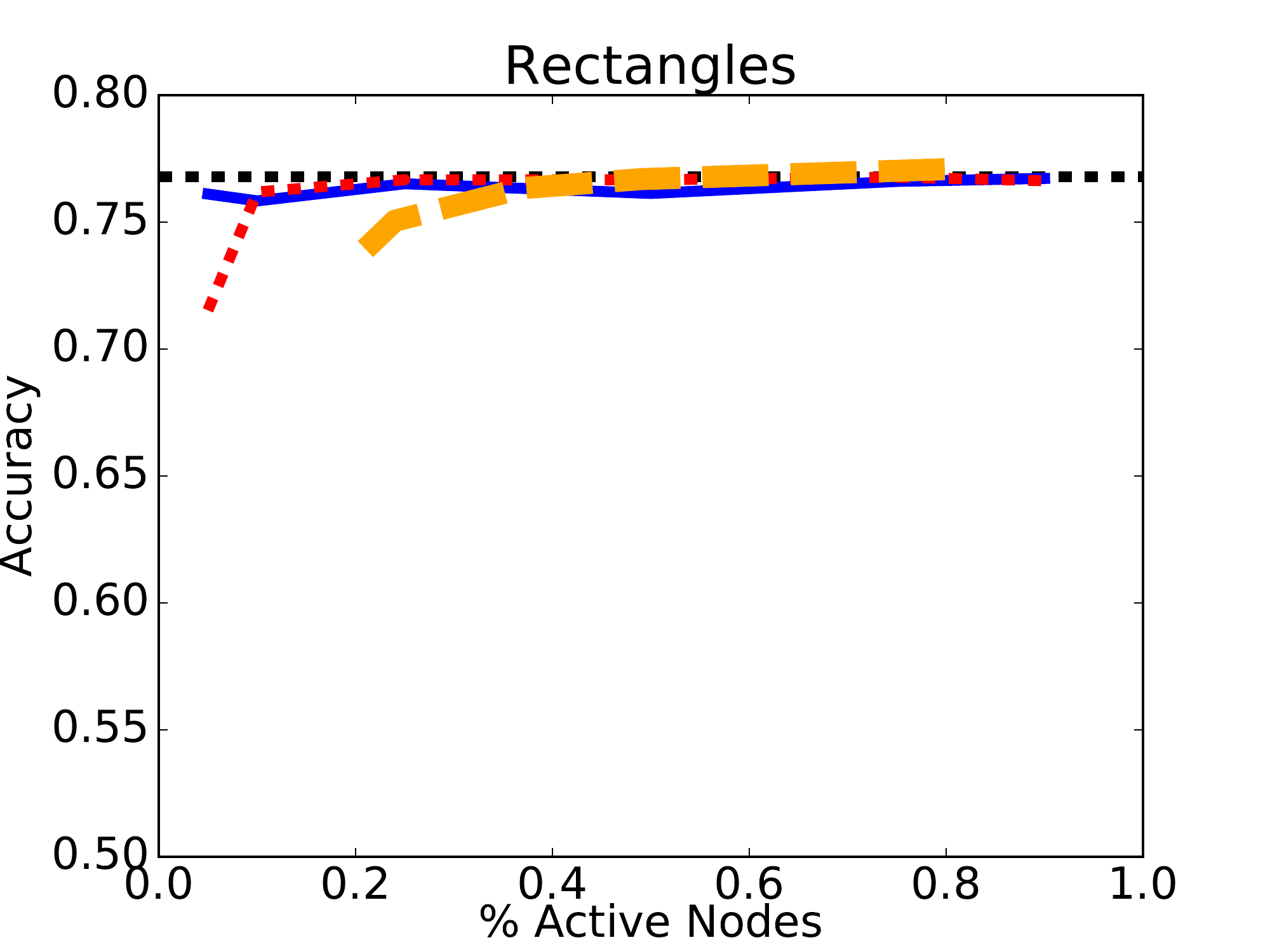} \hspace{-0.19in}}
\mbox{\hspace{-0.185in}
\includegraphics[width=1.9in]{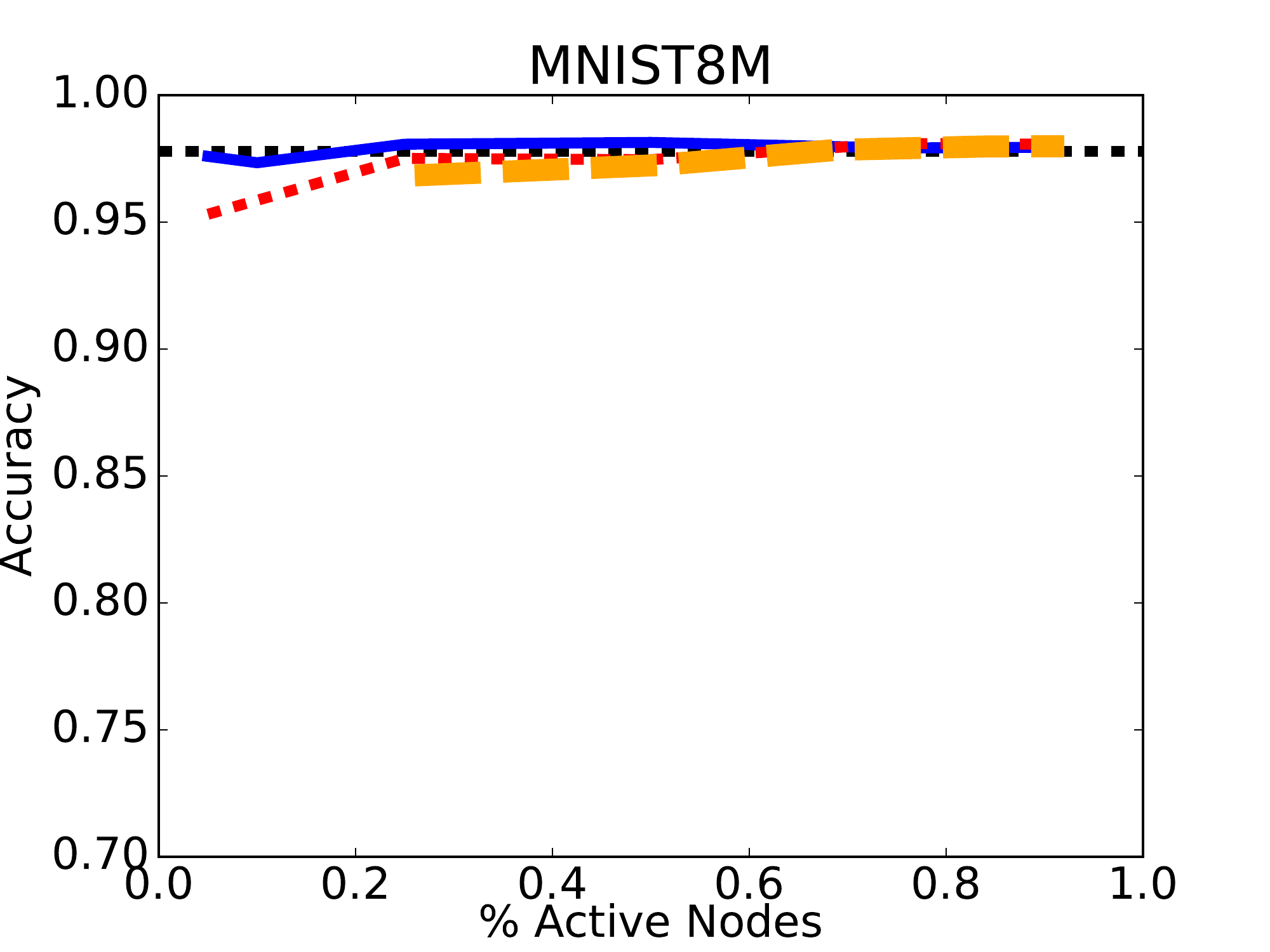} \hspace{-0.15in}
\includegraphics[width=1.9in]{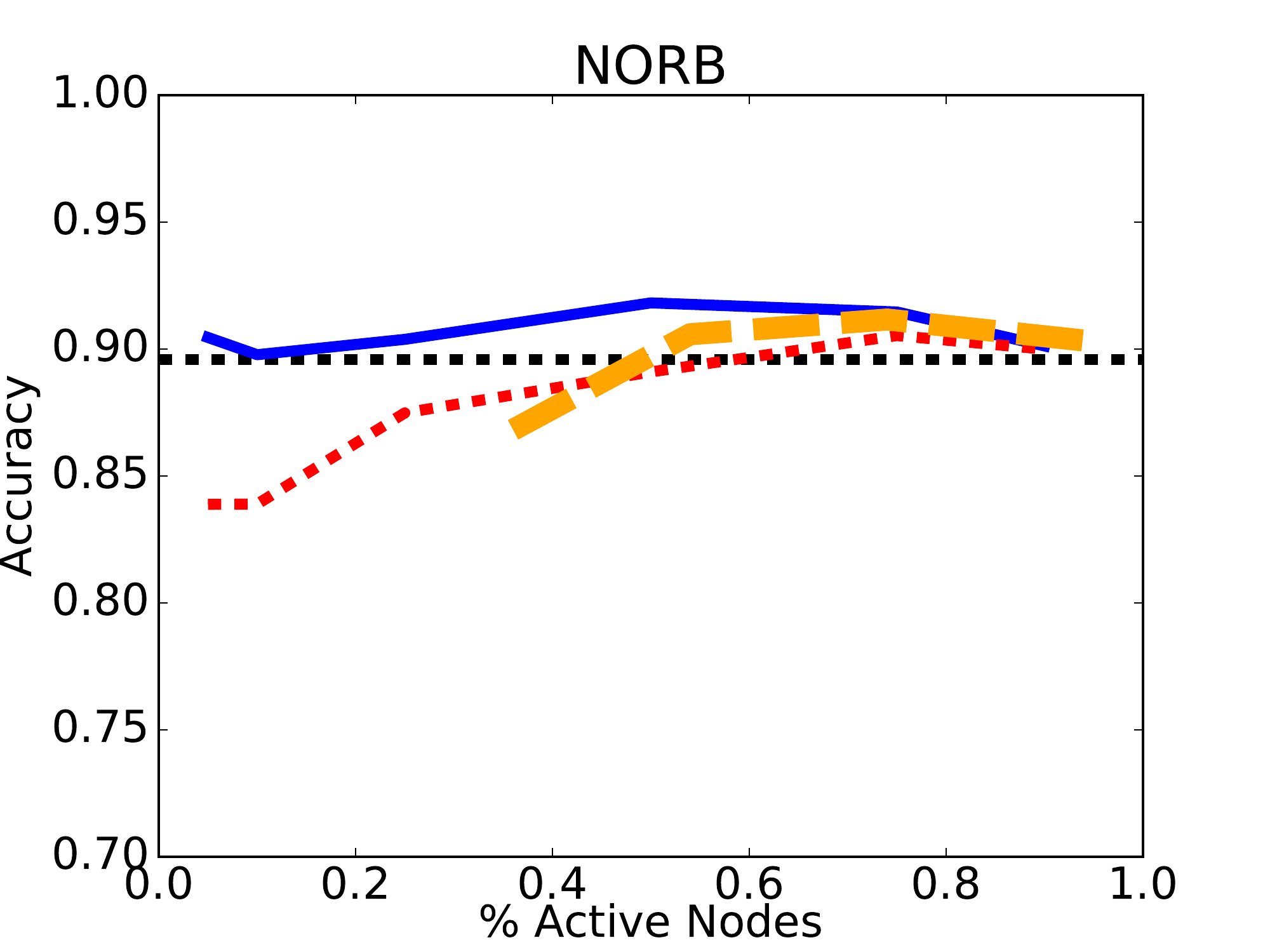} \hspace{-0.19in}
\includegraphics[width=1.9in]{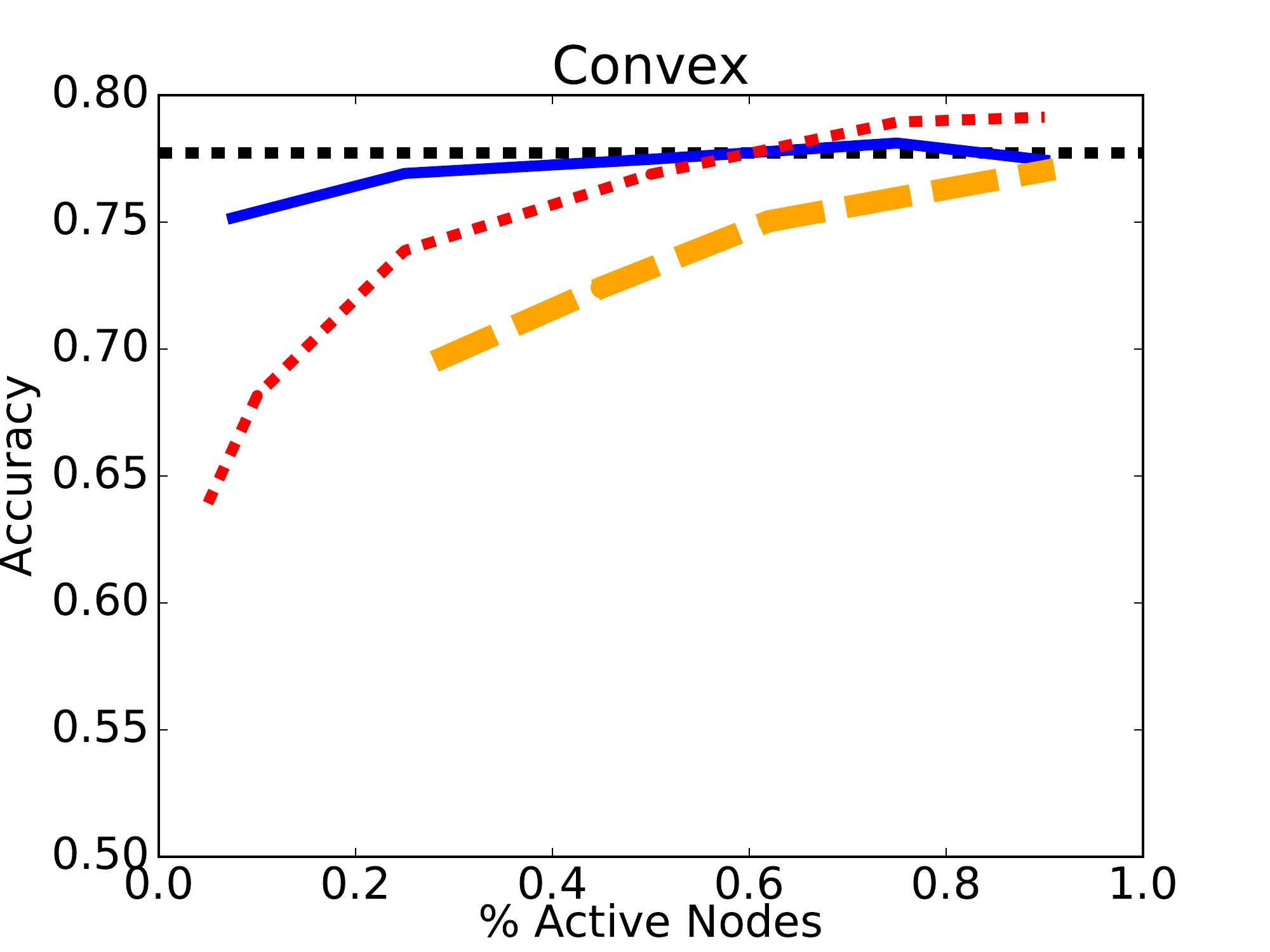} \hspace{-0.19in}
\includegraphics[width=1.9in]{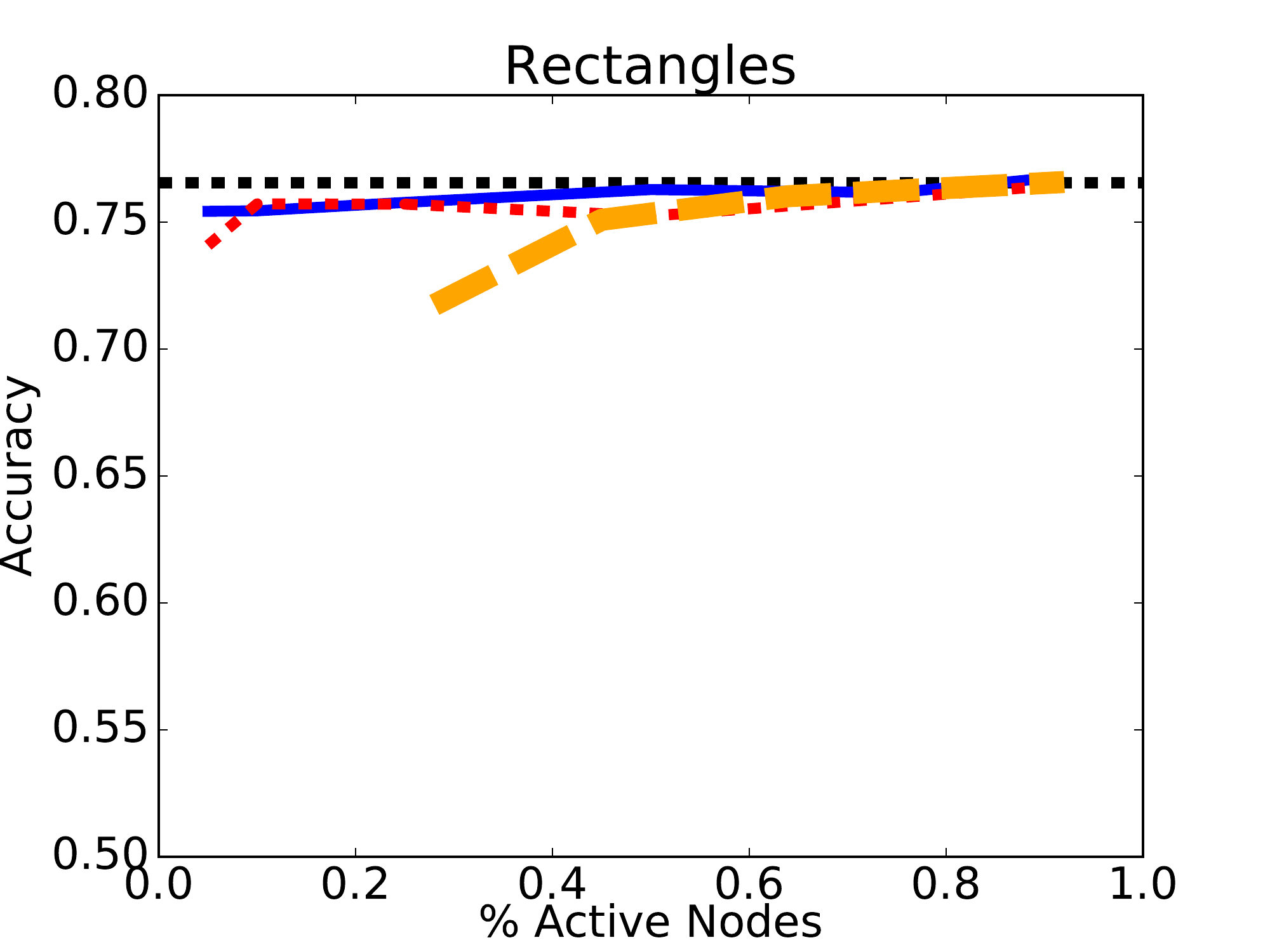} \hspace{-0.19in}}
\end{center}
\caption{Classification accuracy under different levels of active nodes with networks on the MNIST (1st), NORB (2nd), Convex (3rd) and Rectangles (4th) datasets. The standard neural network (dashed black line) is our baseline accuracy. WTA and AD (dashed red + yellow lines) perform the same amount of computation as the standard neural network. Those techniques select nodes with high activations, after full computations, to achieve better accuracy. We compare our LSH approach to determine whether our randomized algorithm achieves comparable performance while reducing the total amount of computation. We do not have data for adaptive dropout at the 5\% and 10\% computation levels because those models diverged when the number of active nodes dropped below 25\%. {\bf 1. Top Panels:} 2 hidden Layers. {\bf 1. Bottom Panels:} 3 hidden Layers }
  \label{LSH_AD_WTA_STD}
\end{figure*}

\subsection{Sustainability}
\label{sec:dropout}
\subsubsection{Experimental Setting}
All of the experiments for our approach and the other techniques were run on a 6-core Intel i7-3930K machine with 16 GB of memory. Our approach uses stochastic gradient descent with Momentum and Adagrad \cite{dean2012large}. Since our approach uniquely selects an active set of nodes for each hidden layer, we focused on a CPU-based approach to simplify combining randomized hashing with neural networks. The ReLU activation function was used for all methods. The learning rate for each approach was set using a standard grid search and ranged between $10^{-2}$ and $10^{-4}$. The parameters for the randomized hash tables were K = 6 bits and L = 5 tables with multi-probe LSH~\cite{lv2007multi} creating a series of probes in each hash tables. We stop early if we find that we have samples enough nodes even before exhausting all buckets. Since we evaluate all levels of selection, in order to increase the percentage of nodes retrieved we increase the number of probes in the buckets. For the experiments, we use a fixed threshold to cap the number of active nodes selected from the hash tables to guarantee the amount of computation is within a certain level. 

\subsubsection{Effect of computation levels}
Figures \ref{LSH_VD_STD}, \ref{LSH_AD_WTA_STD} show the accuracy of each method on neural networks with 2 (Top panel) and 3 (Bottom Panel) hidden layers with the percentage of active nodes ranging from [0.05, 0.10, 0.25, 0.5, 0.75, 0.9]. The standard neural network is our baseline in these experiments and is marked with a dashed black line. Each hidden layer contains 1000 nodes. The x-axis represents the average percentage of active nodes per epoch selected by each technique. Our approach only performs the forward and back propagation steps on the nodes selected in each hidden layer. The other baseline techniques except for Dropout (VD) perform the forward propagation step for each node first to compute all the activations, before setting node activations to zero based on the corresponding algorithm. Thus on VD and our proposal requires a lesser number of multiplications compared to a standard neural network training procedure. 

Figures~\ref{LSH_VD_STD} and \ref{LSH_AD_WTA_STD} summarizes the accuracy of different approaches at various computations levels. From the figures, we conclude the following. 
\begin{itemize}
    \item The plots are consistent across all data-sets and architectures with different depths.
    \item Our method (LSH) gives the best overall accuracy with the fewest number of active nodes. The fact that our approximate method is even slightly better than WTA and adaptive dropouts is not surprising, as it is long known that a small amount of random noise leads to better generalization. For examples, see~\cite{srivastava2014dropout}.
    \item As the number of active nodes decreases from 90\% to 5\%, LSH experiences the smallest drop in performance and less performance volatility.
    \item VD experiences the greatest drop in performance when reducing the number of active nodes from 50\% to 5\%.
    \item As the number of hidden layers increases in the network, the performance drop for VD becomes steeper.
    \item WTA performed better than VD when the percentage of active nodes is less than 50\%
    \item The number of multiplications is a constant multiple of the percentage of active nodes in a hidden layer. The parameters, alpha, and beta determine how many nodes are kept active. We used alpha $\alpha$ = 1.0 and beta $\beta$ = [-1.5, -1.0, 0, 1.0, 3.5] for our experiments. Our model diverged when the number of active nodes dropped below 25\%, so we do not have data for Adaptive Dropout at the 5\% and 10\%.
    \item The performance for each method stabilizes, as the computation level approaches 100\%.
\end{itemize}

Lowering the computational cost of running neural networks by running fewer operations reduces the energy consumption and heat produced by the processor. However, large neural networks provide better accuracy and arbitrarily reducing the amount of computation hurts performance. Our experiments show that our method (LSH) performs well at low computation levels, providing the best of both worlds - high performance and low processor computation. This approach is ideal for mobile phones, which have a thermal power design (TDP) of 3-4 Watts, because reducing the processor's load directly translates into longer battery life.

\begin{figure*}[ht]
\begin{center}
\mbox{\hspace{-0.185in}
\includegraphics[width=1.9in]{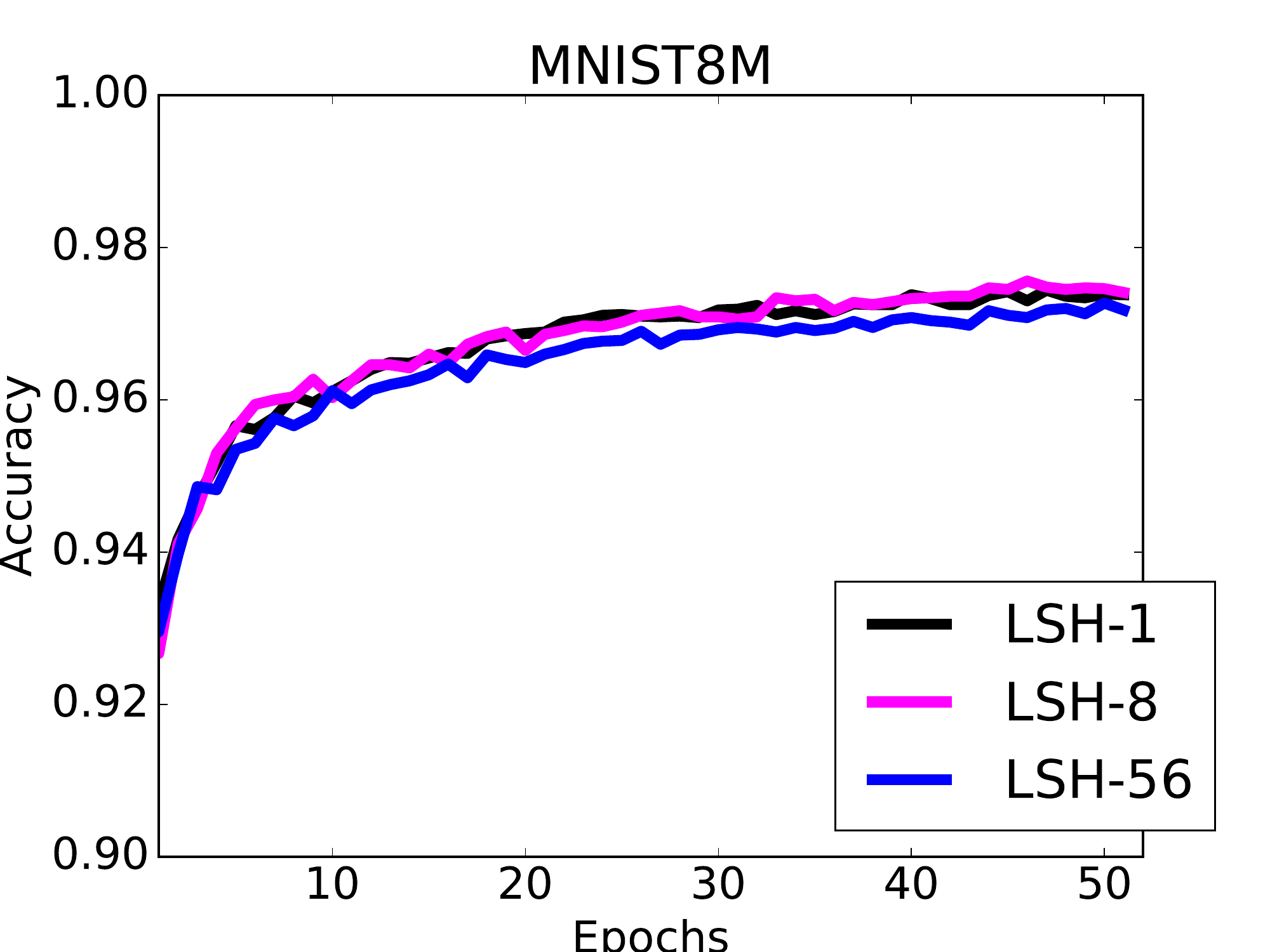} \hspace{-0.15in}
\includegraphics[width=1.9in]{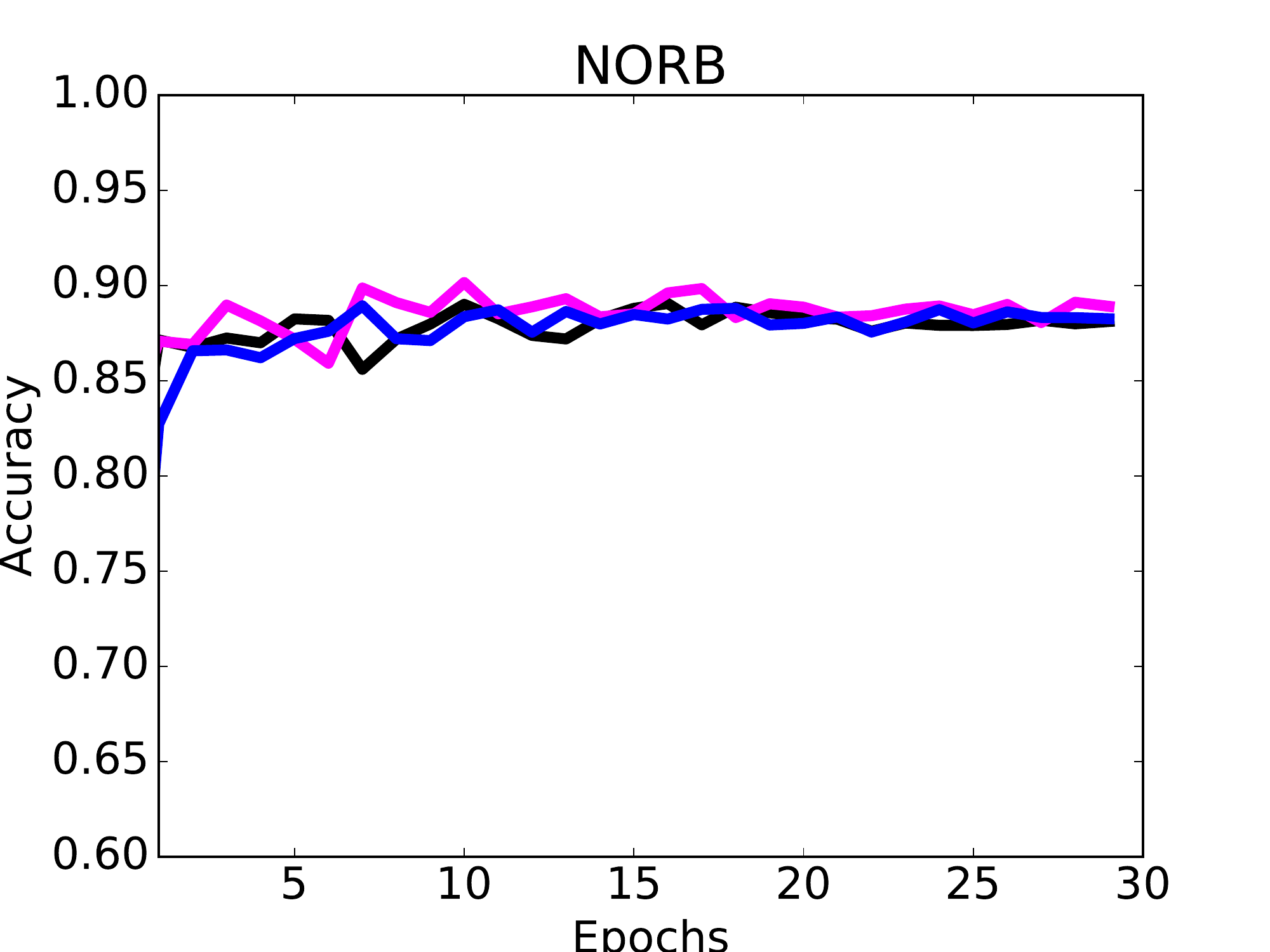} \hspace{-0.19in}
\includegraphics[width=1.9in]{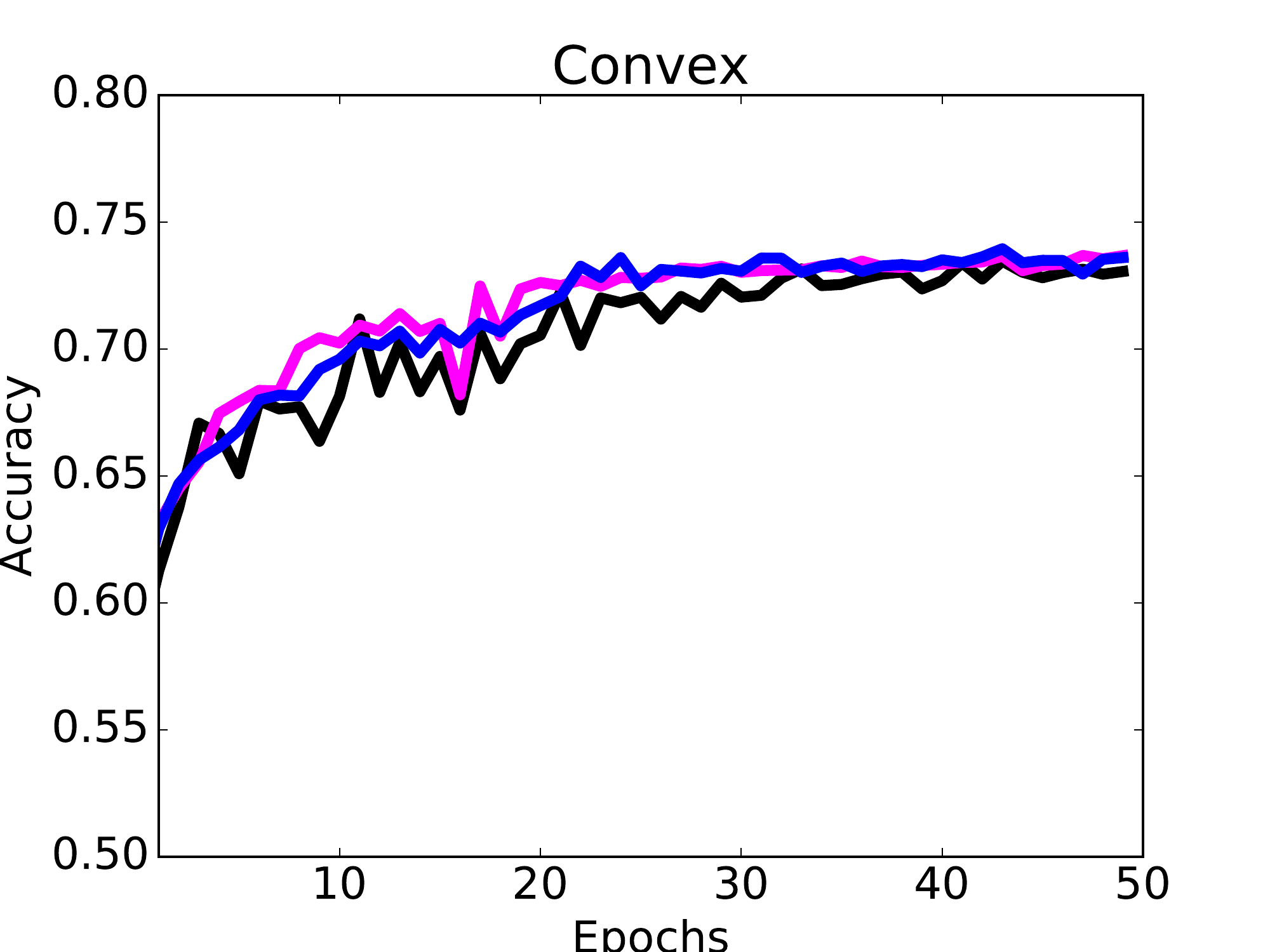} \hspace{-0.19in}
\includegraphics[width=1.9in]{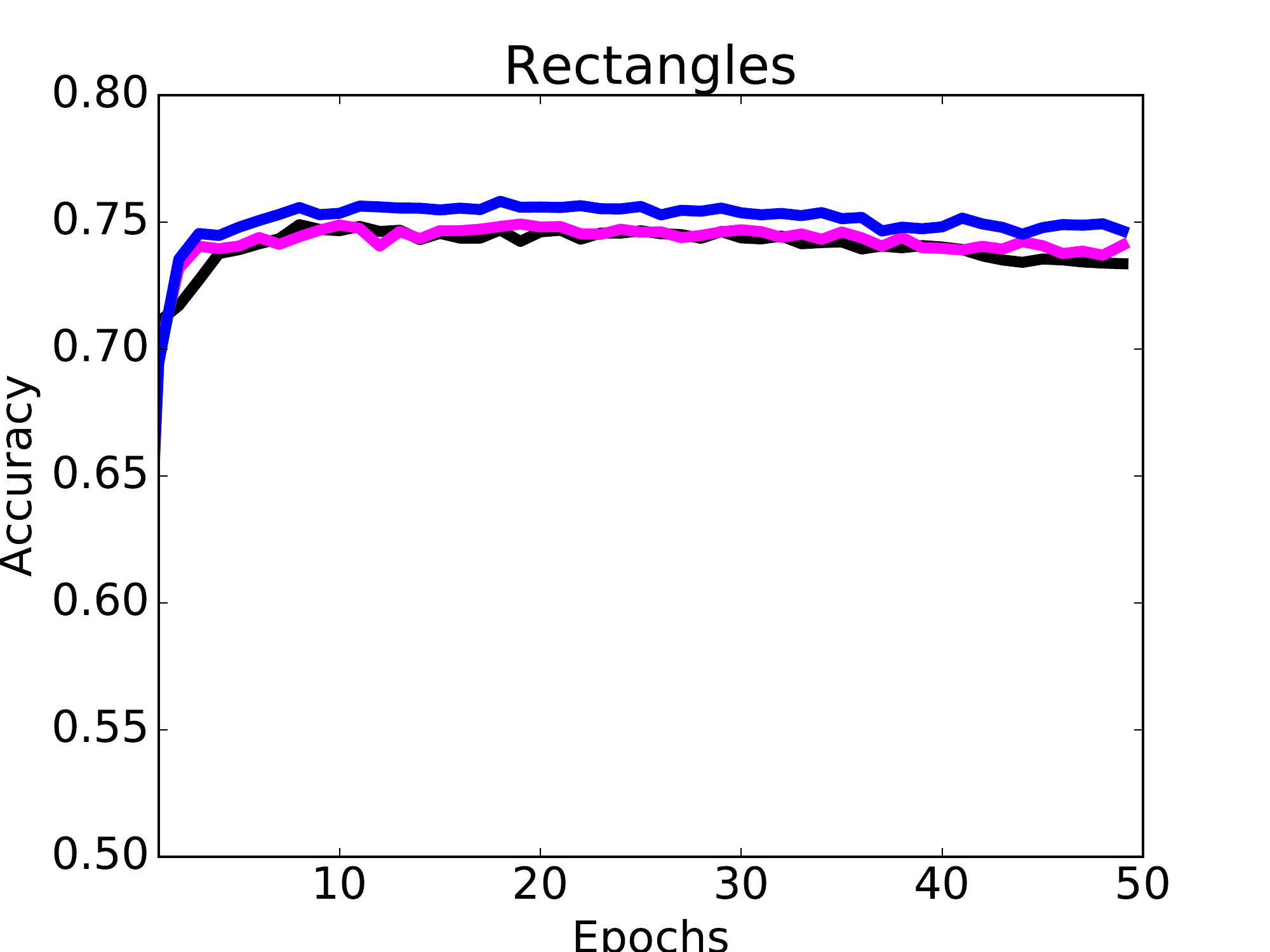} \hspace{-0.19in}
}
\end{center}
\caption{The convergence of our randomized hashing approach (LSH-5\%) over several training epochs using asynchronous stochastic gradient with 1, 8, 56 threads. We used a (3 hidden layer) network on the MNIST (1st), NORB (2nd), Convex (3rd) and Rectangles (4th) datasets. Only 5\% of the standard network's computation was performed in this experiment.}
  \label{asgd_accuracy}
\end{figure*}

\begin{figure*}[ht]
\begin{center}
\mbox{\hspace{-0.185in}
\includegraphics[width=1.9in]{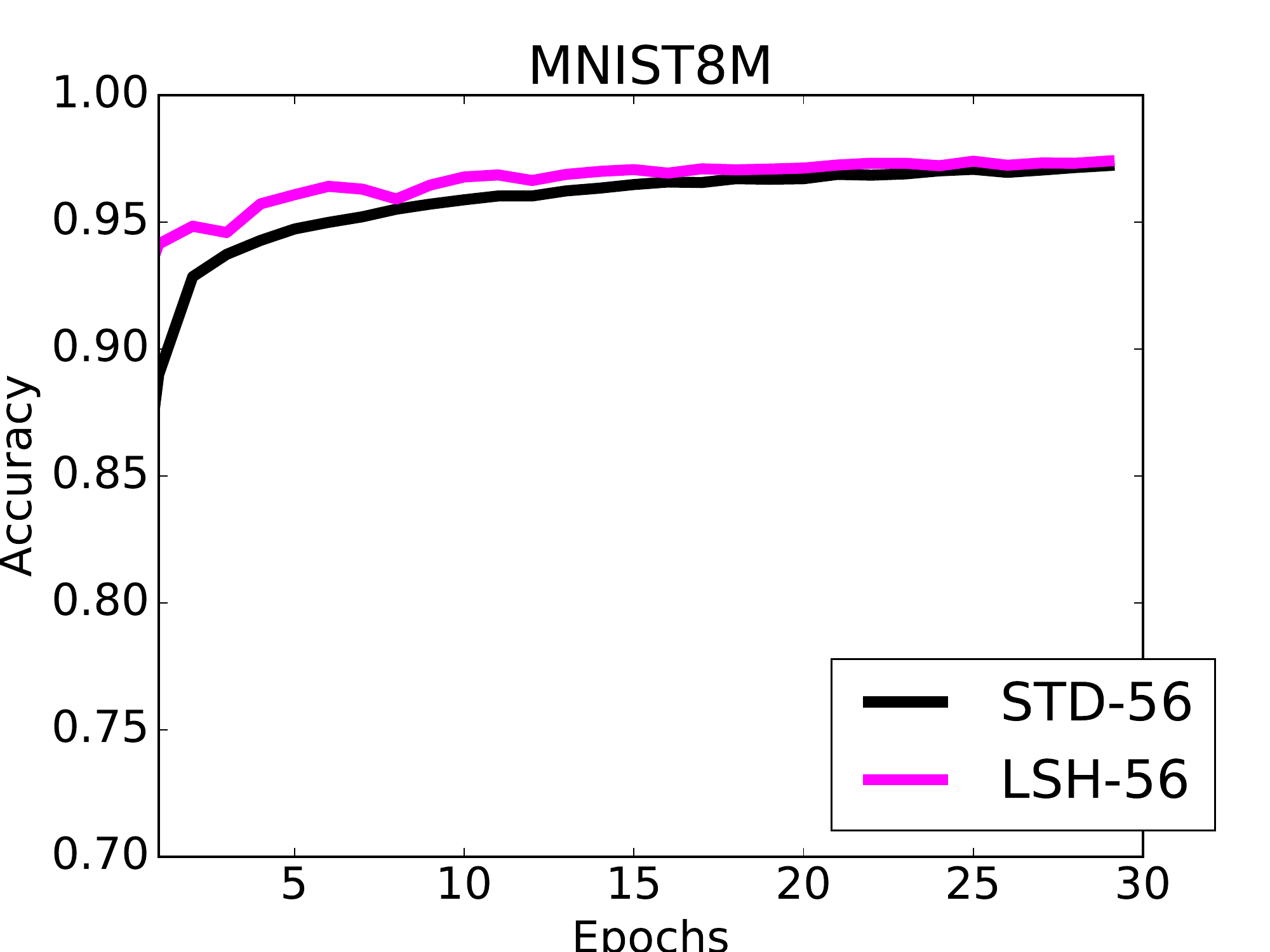} \hspace{-0.15in}
\includegraphics[width=1.9in]{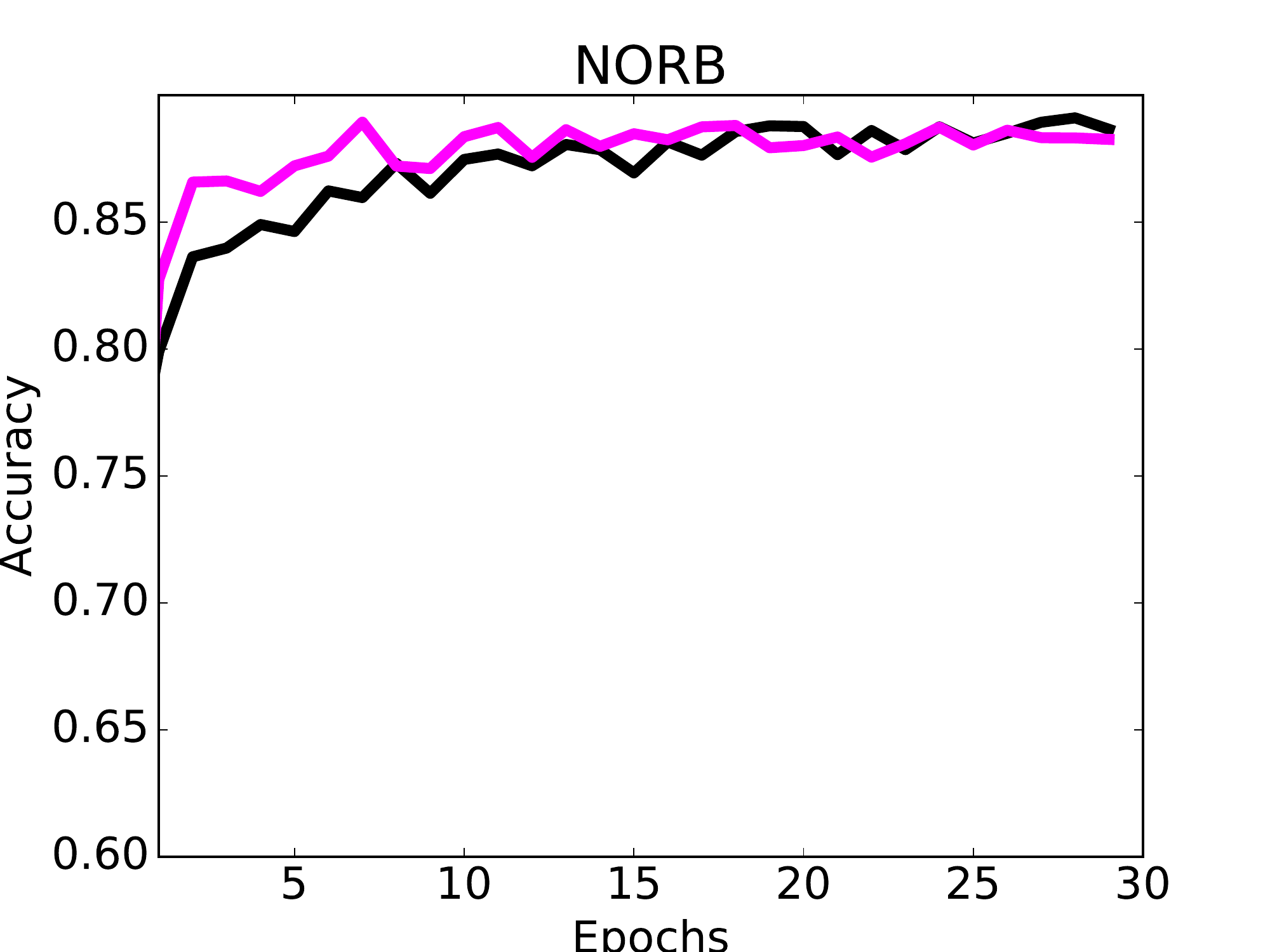}  \hspace{-0.19in}
\includegraphics[width=1.9in]{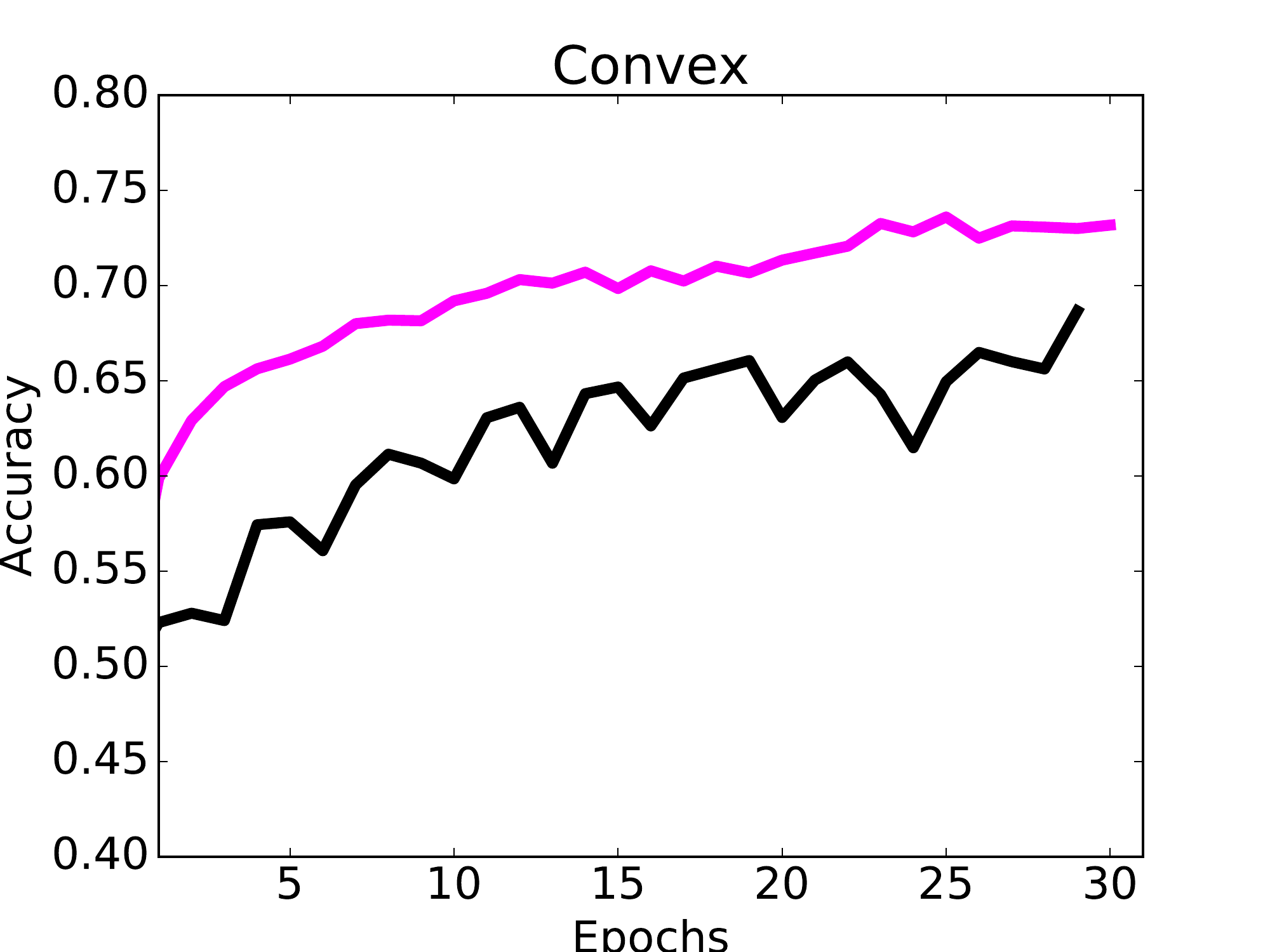} \hspace{-0.19in}
\includegraphics[width=1.9in]{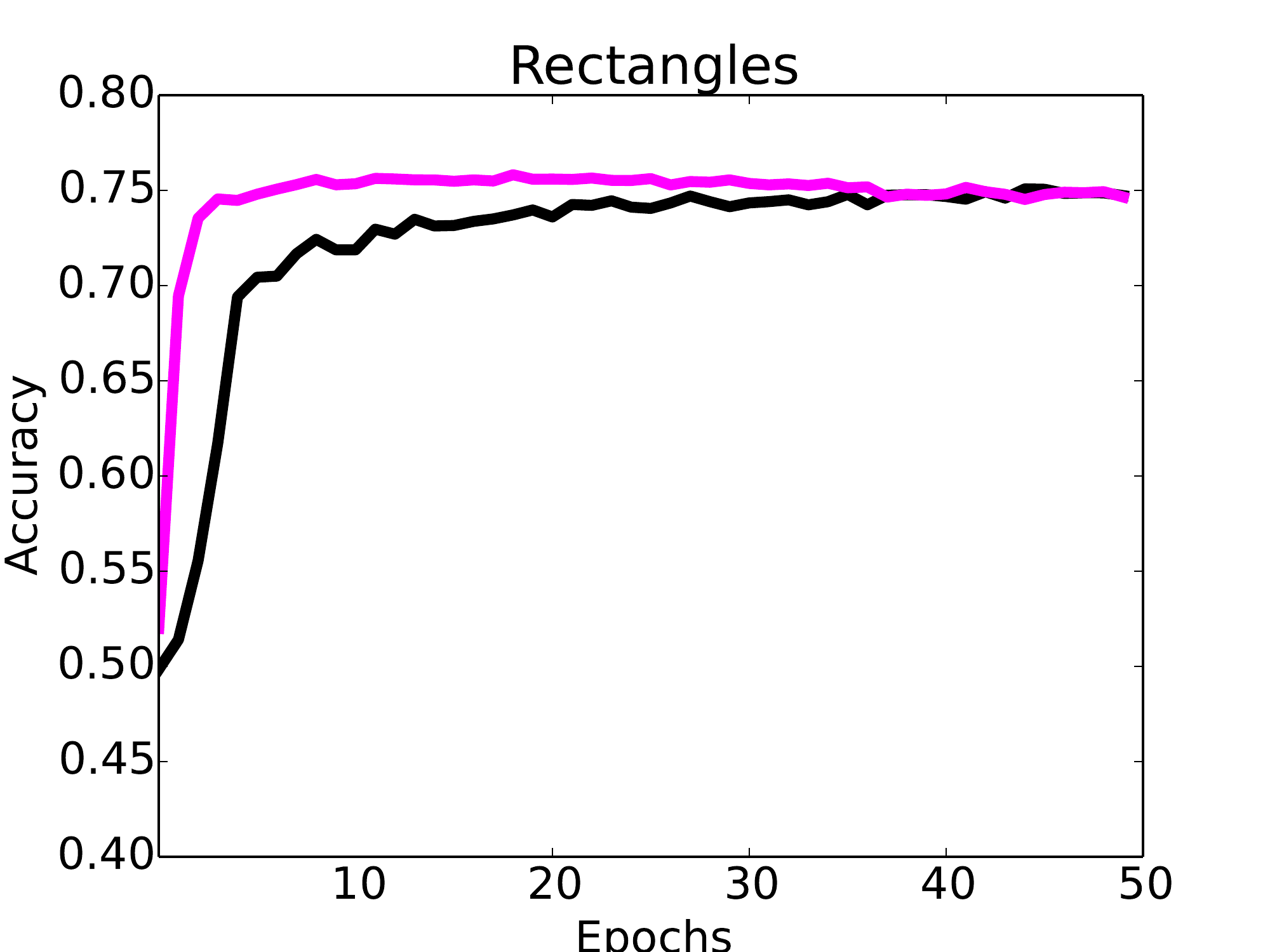} \hspace{-0.19in}
}
\end{center}
\caption{Performance comparison between our randomized hashing approach and a standard network using asynchronous stochastic gradient descent on an Intel Xeon ES-2697 machine with 56-cores. We used (3 hidden layer) networks on MNIST (1st), NORB (2nd), Convex (3rd) and Rectangles (4th). All networks were initialized with the same settings for this experiment.}
  \label{asgd_lsh_std}
\end{figure*}

\begin{figure*}[ht]
\begin{center}
\mbox{\hspace{-0.185in}
\includegraphics[width=1.9in]{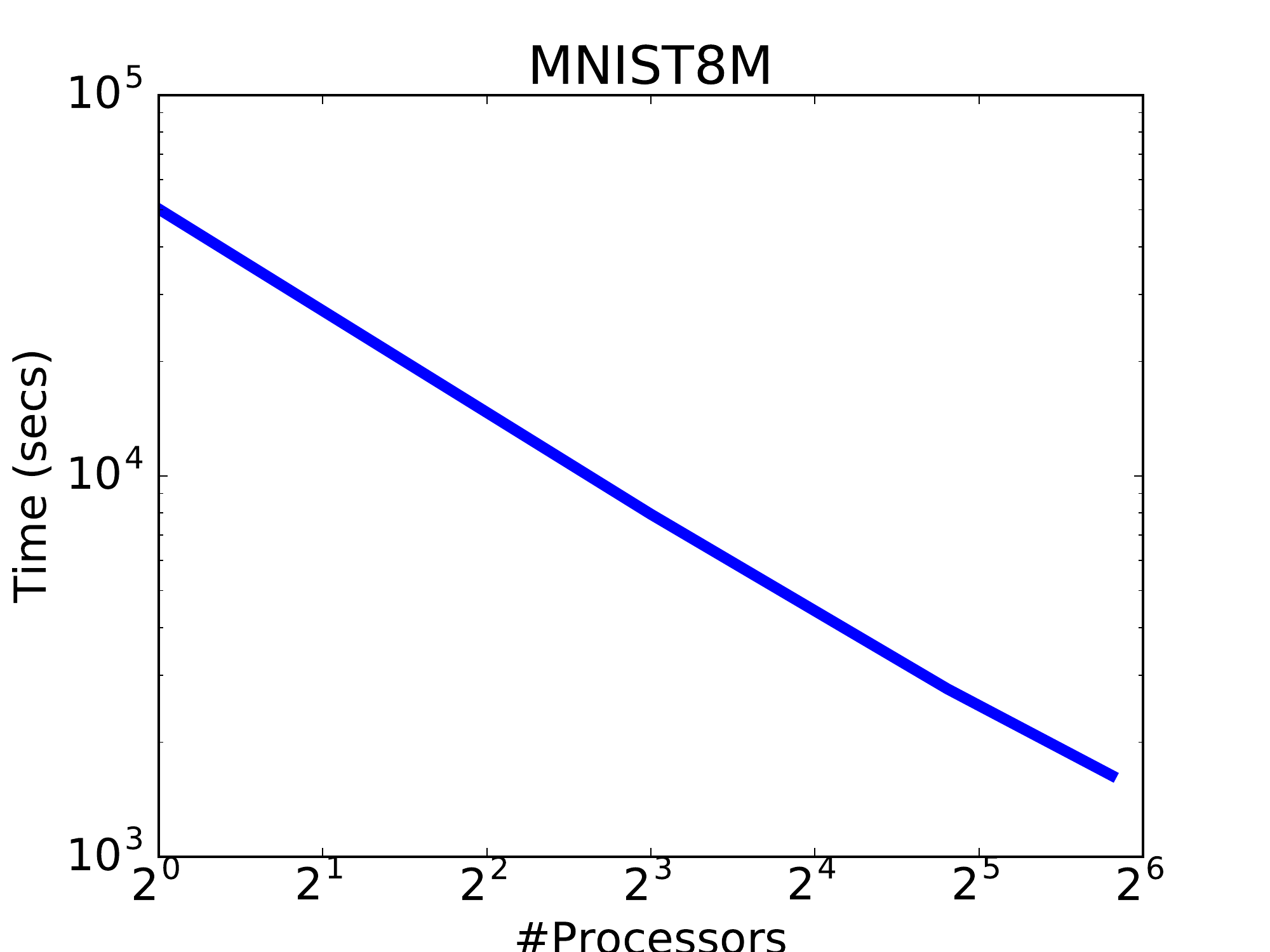} \hspace{-0.15in}
\includegraphics[width=1.9in]{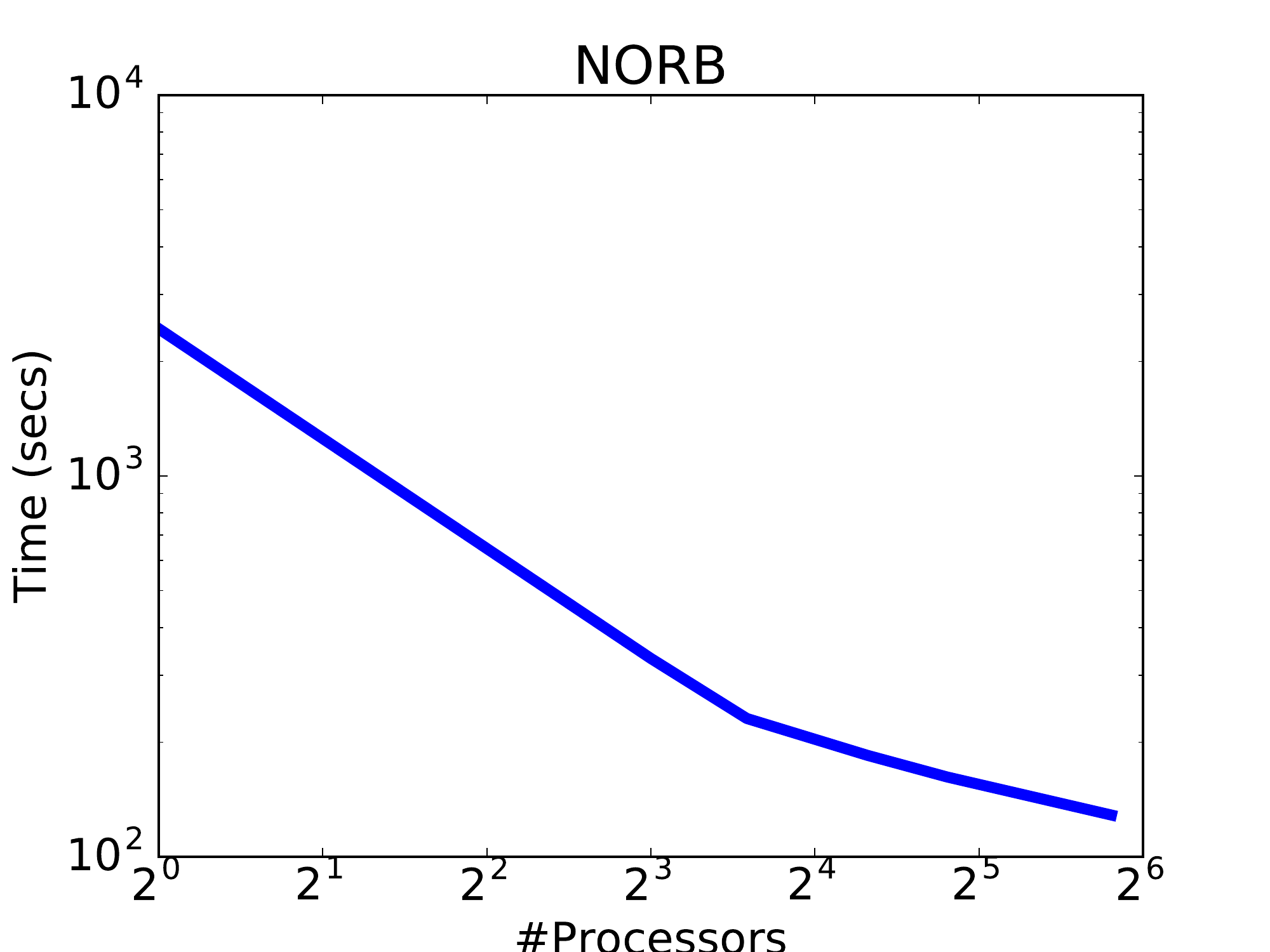} \hspace{-0.19in}
\includegraphics[width=1.9in]{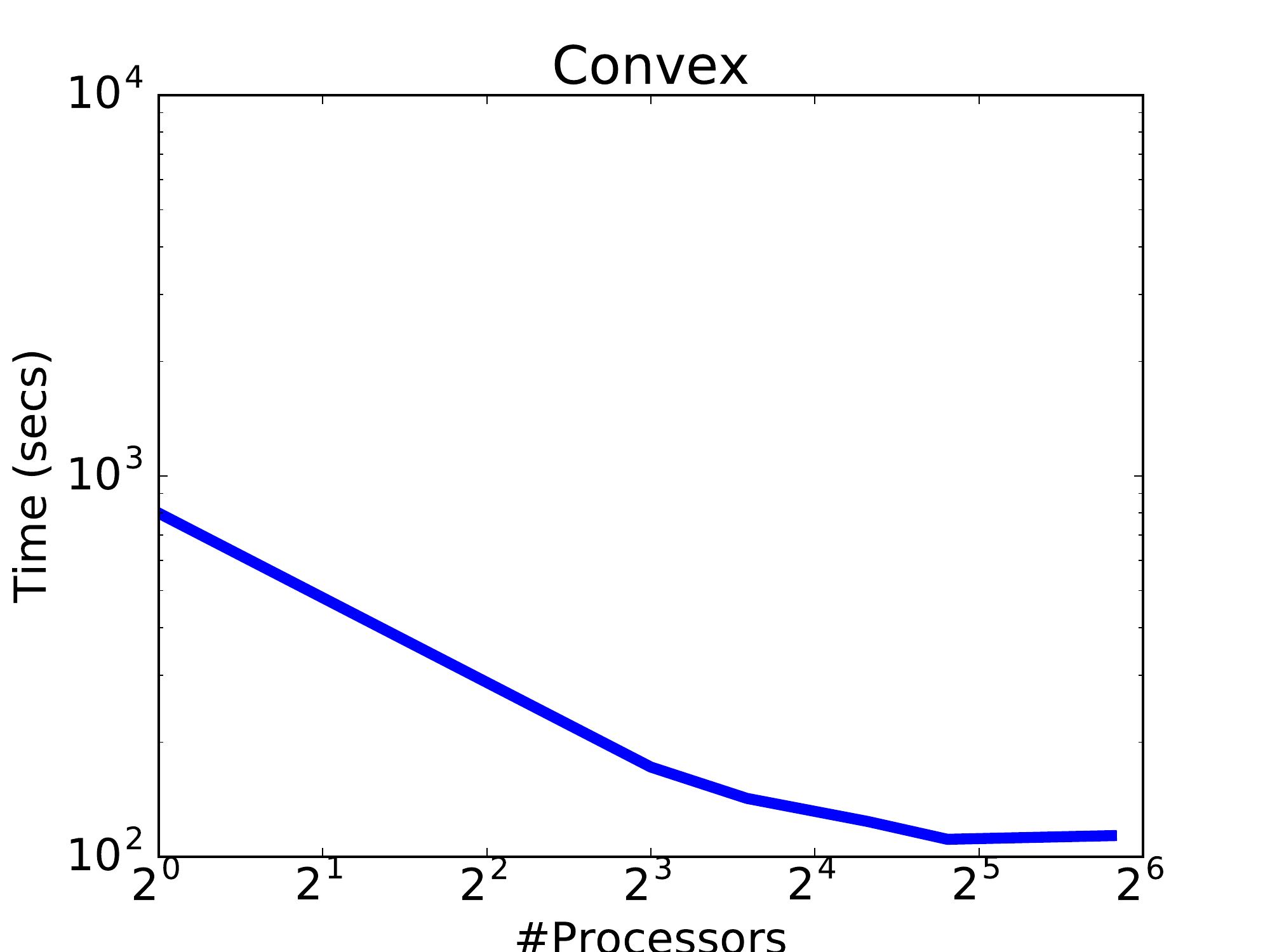} \hspace{-0.19in}
\includegraphics[width=1.9in]{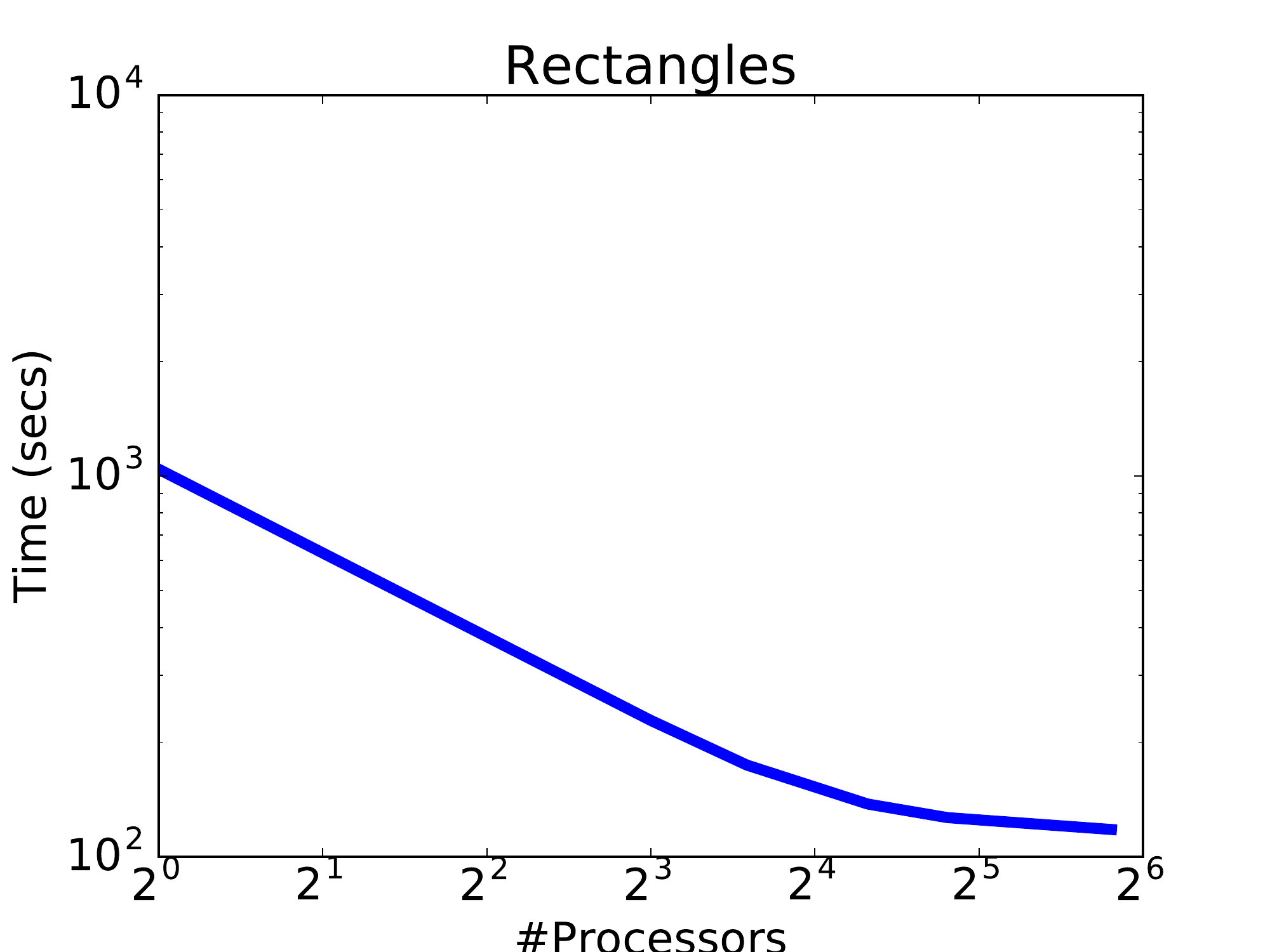} \hspace{-0.19in}
}
\end{center}
\caption{The wall-clock per epoch for our approach (LSH-5\%) gained by using asynchronous stochastic gradient descent. We used a (3 hidden layer) network on the MNIST (1st), NORB (2nd), Convex (3rd) and Rectangles (4th). We see smaller performance gains with the Convex and Rectangles datasets because there are not enough training examples to use of all of the processors effectively. Only 5\% of the standard network's computation was performed in this experiment.}
\label{asgd_performance}
\end{figure*}

\subsection{Scalability}
\label{sec:ASGD}
\subsubsection{Experimental Setting}
We now show experiments to demonstrate the scalability of our approach to large-scale, distributed computing environments. Specifically, we are testing if our approach maintains accuracy and improves training time, as we increase the number of processors. We use asynchronous stochastic gradient descent with momentum and adagrad \cite{dean2012large, recht2011hogwild}. Our implementation runs the same model, with similar initializations, on multiple training examples concurrently. The gradient is applied without synchronization or locks to maximize performance. We run all of the experiments on an Intel Xeon ES-2697 machine with 56 cores and 256 GB of memory. The ReLU activation was used for all models, and the learning rate ranged between $10^{-2}$ and $10^{-3}$.

\subsubsection{Results with different numbers of processors}
Figure \ref{asgd_accuracy} shows how our method performs with asynchronous stochastic gradient descent (ASGD) using only 5\% of the neurons of a full-sized neural network. The neural network has three hidden layers, and each hidden layer contains 1000 neurons. The x-axis represents the number of epochs completed, and the y-axis shows the test accuracy for the given epoch. We compare how our model converges with multiple processors working concurrently. Since our ASGD implementation does not use locks, it depends on the sparsity of the gradient to ensure the model converges and performs well \cite{recht2011hogwild}. From our experiments, we see that our method converges at a similar rate and obtains the same accuracy regardless of the number of processors running ASGD.

Figure \ref{asgd_performance} illustrates how our method scales with multiple processors. The inherent sparsity of our randomized hashing approach reduces the number of simultaneous updates and allows for more asynchronous models without any performance penalty. We show the corresponding drop in wall-clock computation time per epoch while adding more processors. We achieve roughly a 31x speed up while running ASGD with 56 threads.

\subsubsection{ASGD Performance Comparison with Standard Neural Network}
Figure \ref{asgd_lsh_std} compares the performance of our LSH approach against a standard neural network (STD) when running ASGD with 56-cores. We used a standard network with a mini-batch of size 32. We clearly out-perform the standard network for all of our experimental datasets. However, since there is a large number of processors applying gradients to the parameters, those gradients are constantly being overridden, preventing ASGD from converging to an optimal local minimum. Our approach produces a spare gradient that reduces the conflicts between the different processors, while keeping enough valuable nodes for ASGD to progress towards the local minimum efficiently. 

From Figures \ref{asgd_accuracy}, \ref{asgd_lsh_std} and \ref{asgd_performance}, we conclude the following. 
\begin{enumerate}
\item The gradient updates are quite sparse with 5\% LSH and running multiple ASGD updates in parallel does not affect the convergence rate of our hashing based approach. Even when running 56 cores in parallel, the convergence is indistinguishable from the sequential update (1 core) on all the four datasets.
\item If we instead run vanilla SGD in parallel, then the convergence is effected. The convergence is in general slower compared the sparse 5\% LSH. This slow convergence is due to dense updates which leads to overwrites. Parallelizing dense updates affects the four datasets differently. For convex dataset, the convergence is very poor.
\item As expected, we obtain near-perfect decrease in the wall clock times with increasing the number of processors with LSH-5\%. Note, if there are too many overwrites, then atomic overwrites are necessary, which will create additional overhead and hurt the parallelism. Thus, near-perfect scalability also indicates fewer gradient overwrites. 
\item On the largest dataset - MNIST8M, the running time per epoch for the 1-core implementation is 50,254 seconds. The 56-core implementation runs in 1,637 seconds. Since the convergence is not affected, this amounts to about a 31x speedup in the training process with 56 processors. 
\item We see that the gains of parallelism become flat with the Convex and Rectangle datasets, especially while using a large number of processors. This poor scaling is because the two datasets have fewer training examples, and so there is less parallel work for a large number of processor, especially when working with 56 cores. We do not see such behaviors with MNIST8M which has around 8 million training examples.
\end{enumerate}

\section{Discussions}
Machine learning with a huge parameter space is becoming a common phenomenon. SGD remains the most promising algorithm for optimization due its effectiveness and simplicity. SGD that updates from a single data point is unlikely to change the entire parameter space significantly. Each SGD iteration is expected to change only a small set (the active set) of parameters, depending on the current sample. This sparse update occurs because there is not enough information in one data point. However, identifying that small set of active parameters is a search problem, which typically requires computations on the order of parameters. We can exploit the rich literature in approximate query processing to find this active set of parameters efficiently. Of course, the approximate active set contains a small amount of random noise, which is often good for generalization. Sparsity and randomness enhance data parallelism because sparse, and random SGD updates are unlikely to overwrite each other. 

Approximate query processing already sits on decades of research from the systems and database community, which makes the algorithm scalable over a variety of distributed systems. Thus, we can forget about the systems challenges by reformulating the machine learning problem into an approximate query processing problem and levering the ideas and implementation from a very rich literature. We have demonstrated one concrete evidence by showing how deep networks can be scaled-up. We believe that such integration of sparsity with approximate query processing, which is already efficient over different systems, is the future of large-scale machine learning.

\section{Conclusion and Future Work}
Our randomized hashing approach is designed to minimize the amount of computation necessary to train a neural network effectively. Training neural networks requires a significant amount of computational resources and time, limiting the scalability of neural networks and its sustainability on embedded, low-powered devices. Our results show that our approach performs better than the other methods when using only 5\% of the nodes executed in a standard neural network. The implication is that neural networks using our approach can run on mobile devices with longer battery life while maintaining performance. We also show that due to inherent sparsity our approach scales near-linearly with asynchronous stochastic gradient descent running with multiple processors.

Our future work is to optimize our approach for different mobile processors. We have many choices to choose from including the Nvidia Tegra and Qualcomm Snapdragon platforms. Recently, Movidius released plans for a specialized deep learning processor called the Myriad 2 VPU (visual processing unit) \cite{barry2015always}. The processor combines the best aspects of GPUs, DSPs, and CPUs to produce 150 GFLOPs and while using about 1 W of power. We plan to leverage the different platform architectures to achieve greater efficiency with our randomized hashing approach.

\section{Acknowledgments}
The work of Ryan Spring was supported from NSF Award 1547433. This work was supported by Rice Faculty Initiative Award 2016.

\bibliographystyle{abbrv}
\bibliography{vldb}
\balance

\end{document}